\documentclass[12pt]{article}
\usepackage{fullpage,archive}
\usepackage{times}
\usepackage{amsmath,url,amsthm,amssymb,amsbsy,fancybox}
\usepackage{algorithmic,algorithm}
\usepackage{subfigure,epsfig,graphicx}
\usepackage{multirow}

\usepackage{epsfig}
\usepackage{graphicx}
\usepackage{wrapfig}
\usepackage{amsfonts}
\usepackage{amsthm}
\usepackage{multirow}
\usepackage{amsmath}
\usepackage{rotating}
\usepackage{tabularx}
\usepackage{subfigure}

\def\argmin{\mathop{\text{arg\,min}}}
\def\argmax{\mathop{\text{arg\,max}}}
\def\Tr{\mathop{\text{Tr}}}
\let\hat\widehat
\let\tilde\widetilde

\newcommand{\bm}[1]{\mbox{\boldmath$#1$\unboldmath}}
\newcommand{\mysgn}{\textrm{sign}}

\newcommand{\lmult}{$\ell_1/\ell_2$}

\newcommand{\bx}{{\mathbf x}}
\newcommand{\by}{{\mathbf y}}

\newcommand{\bX}{{\mathbf X}}
\newcommand{\bY}{{\mathbf Y}}

\newcommand{\bB}{{\mathbf B}}
\newcommand{\bW}{{\mathbf W}}
\newcommand{\bZ}{{\mathbf Z}}
\newcommand{\bA}{{\mathbf A}}

\newcommand{\bV}{{\mathbf V}}

\newcommand{\bb}{{\boldsymbol \beta}}
\newcommand{\ba}{{\boldsymbol \alpha}}

\newcommand{\bg}{{\boldsymbol \Theta}}

\newtheorem{lemma}{Lemma}

\newtheorem{theorem}{Theorem}

\usepackage{fancyhdr}
\fancypagestyle{rcsfooters}{%

}

\title{Graph-Structured Multi-task Regression and an Efficient Optimization Method for General Fused Lasso}

\author{Xi Chen$^{1}$ \and Seyoung Kim$^{1}$ \and Qihang Lin$^{2}$ \and Jaime G. Carbonell$^{1}$ \and  Eric P. Xing$^{1}$
\thanks{To whom correspondence should be addressed: \textsf{epxing@cs.cmu.edu}}}

\date{}

\abstract{ We consider the problem of learning a structured
multi-task regression, where the output consists of multiple
responses that are related by a graph and the correlated response
variables are dependent on the common inputs in a sparse but
synergistic manner. Previous methods such as \lmult-regularized
multi-task regression assume that all of the output variables are
equally related to the inputs, although in many real-world problems,
outputs are related in a complex manner. In this paper, we propose
graph-guided fused lasso (GFlasso) for structured multi-task
regression that exploits the graph structure over the output
variables. We introduce a novel penalty function based on fusion
penalty to encourage highly correlated outputs to share a common set
of relevant inputs. In addition, we propose a simple yet efficient
proximal-gradient method for optimizing GFlasso that can also be
applied to any optimization problems with a convex smooth loss and
the general class of fusion penalty defined on arbitrary graph
structures. By exploiting the structure of the non-smooth ``fusion
penalty'', our method achieves a faster convergence rate than the
standard first-order method, sub-gradient method, and is
significantly more scalable than the widely adopted second-order
cone-programming and quadratic-programming formulations. In
addition, we provide an analysis of the consistency property of the
GFlasso model. Experimental results not only demonstrate the
superiority of GFlasso over the standard lasso but also show the
efficiency and scalability of our proximal-gradient method. }


\begin{document}

\maketitle

\section{Introduction}

%

In multi-task learning, we are interested in learning multiple
related tasks jointly by analyzing data from all
of the tasks at the same time instead of considering each task
individually \cite{Yu:2005,Zhang:2008,Obozinski:09}. When data are
scarce, it is greatly advantageous to borrow the information in the
data from other related tasks to learn each task more effectively.

In this paper, we consider a multi-task regression problem, where
each task is to learn a functional mapping from a high-dimensional
input space to a continuous-valued output space and only a small
number of input covariates are relevant to the output. Furthermore,
we assume that the outputs are related in a complex manner, and that
this output structure is available as prior
knowledge in the form of a graph. Given this setting, it is
reasonable to believe that closely related outputs tend to share a
common set of relevant inputs. Our goal is to recover this
structured sparsity pattern in the regression coefficients shared
across tasks related through a graph.

When the tasks are assumed to be equally related to inputs without
any structure, a mixed-norm regularization such as \lmult-  and
$\ell_1/\ell_\infty$-norms has been used to find inputs relevant to
all of the outputs jointly \cite{Obozinski:09, Turlach:05}. However,
in many real-world problems, some of the tasks are often more
closely related and more likely to share common relevant covariates
than other tasks. Thus, it is necessary to take into account the
complex correlation structure in the outputs for a more effective
multi-task learning. For example, in genetic association analysis,
where the goal is to discover few genetic variants or single
neucleotide polymorphisms (SNPs) out of millions of SNPs (inputs)
that influence phenotypes (outputs) such as gene expression
measurements \cite{brem:2008}, groups of genes in the same pathways
are more likely to share common genetic variants affecting them than
other genes. In a neuroscience application, an
$\ell_1/\ell_\infty$-regularized multi-task regression has been used
to predict neural activities (outputs) in brain in response to words
(inputs) \cite{Liu:09}. Since neural activities in the brain are
locally correlated, it is necessary to take into account this local
correlation in different brain regions rather than assuming that all
regions share a similar response as in \cite{Liu:09}. A similar
problem arises in stock prediction where some of the stock prices
are more highly correlated than others \cite{Gho:09}.

The main contributions of this paper are two-fold. First, we propose
a structured regularized-regression approach called graph-guided
fused lasso (GFlasso) for  sparse multi-task learning problems and
introduce a novel penalty function based on fusion penalty that
encourages tasks related according to the graph to have a similar
sparsity pattern. Second, we propose an efficient optimization
algorithm based on proximal-gradient method that can be used to
solve GFlasso optimization as well as any optimizations with a
convex smooth loss and general fusion penalty which can be defined
on an arbitrary graph structure.

In addition to the standard lasso penalty for overall sparsity
\cite{Tibshirani:96}, GFlasso employs a ``fusion penalty"
\cite{Tibshirani:05} that fuses regression coefficients across
correlated outputs, using the weighted connectivity of the output graph
as a guide. The overall effect of the GFlasso penalty is that it
allows us to identify a parsimonious set
of input factors relevant to dense subgraphs of outputs as
illustrated in Figure \ref{QTL}. To the best of our knowledge, this
work is the first to consider the graph structure over the outputs
in multi-task learning.
We also provide an analysis of the consistency property of the
GFlasso model.



\begin{figure}[t!]
\centering
\begin{tabular}{@{}c@{}c@{}c}
\parbox[l]{1.8cm}{\footnotesize Inputs (SNPs)}
&\multirow{3}{*}{
\parbox[c]{4.0cm}{\includegraphics[scale = 0.25]{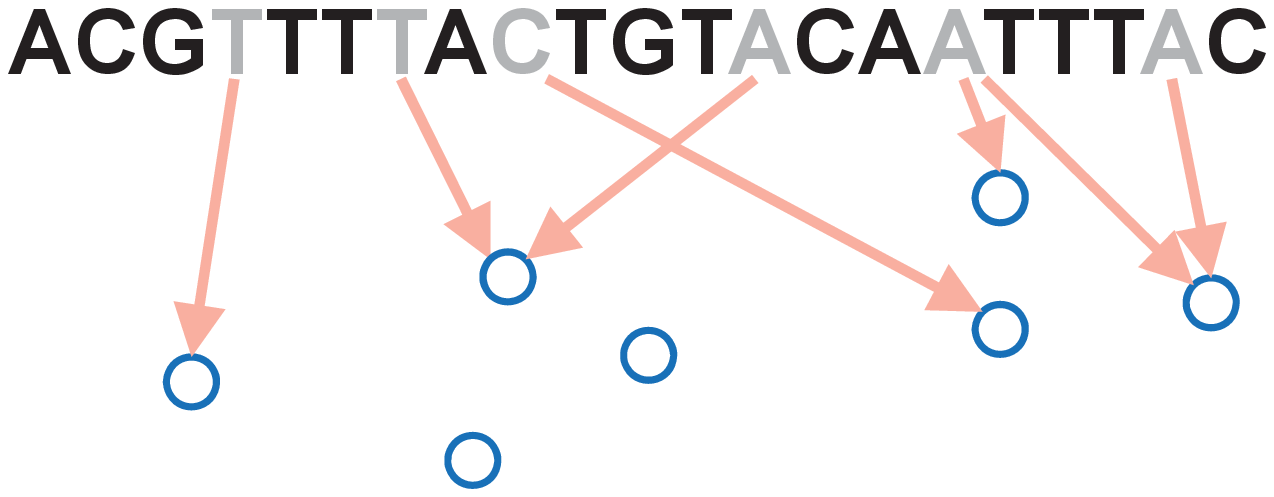}}} &
\multirow{3}{*}{
\parbox[c]{4.0cm}{\includegraphics[scale = 0.25]{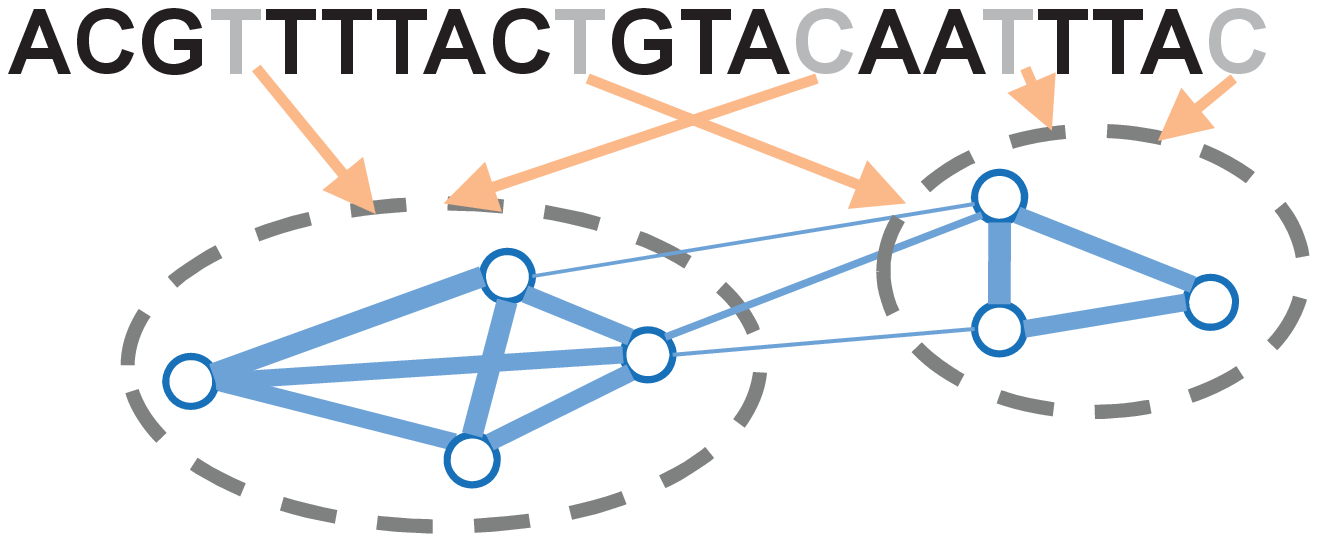}}} \\
\vspace{0pt}
& & \\
\parbox[l]{1.8cm}{\footnotesize Outputs \\(phenotypes)}
&  \\
\vspace{-30pt}
& & \\
& & \\
&  (a)   & (b)
\end{tabular}
\label{QTL}  \caption{Illustrations of multi-task regression
with (a) lasso, (b) graph-guided fused lasso.}
\end{figure}

The fusion penalty that we adopt to construct the GFlasso penalty
has been widely used for sparse learning problems, including fused
lasso \cite{Tibshirani:05}, fused-lasso signal approximator
\cite{Friedman:07}, and network learning \cite{Kolar:09}. However,
because of the non-separability of the fusion penalty function,
developing a fast optimization algorithm has remained a challenge.
The available optimization methods include formulating the problem
as second-order cone programming (SOCP) or quadratic programming
(QP) and solving them by interior-point methods (IPM)
\cite{Tibshirani:05}, but these approaches are computationally
expensive even for problems of moderate size. In order to reduce the
computational cost, when the fusion penalty is defined on a chain or
two-way grid structure over inputs, the pathwise coordinate
optimization has been applied \cite{Friedman:07}. However, when the
fusion penalty is defined on a general graph, this method cannot be
easily applied because of the relatively complex sub-gradient
representation. In addition, as pointed out in \cite{Tibshirani:05},
this method ``is not guaranteed to yield exact solution'' when the
design matrix is not orthogonal. Very recently, an unpublished
manuscript \cite{Holger:09} proposed a different algorithm which
reformulates the problem as a maximum flow problem. However, this
algorithm works only when the dimension is less than the sample size
and hence is not applicable to high-dimensional sparse learning
problems. In addition, it lacks any theoretical guarantee of the
convergence.


In order to make the optimization efficient and scalable, a
first-order method (using only gradient) is desired. In this paper,
we propose a proximal-gradient method to solve GFlasso. More
precisely, by exploiting the structure of the non-smooth fusion
penalty, we introduce its smooth approximation and  optimize this
approximation based on the accelerated gradient method in
\cite{Nesterov:05}. It can be shown that our method achieves the
convergence rate of $O(\frac{1}{\epsilon})$, where the $\epsilon$ is
the desired accuracy.  Our method is significantly faster than the
most natural first-order method, subgradient method \cite{Ber:99}
with $O(\frac{1}{\epsilon^2})$ convergence rate and more scalable by
orders of magnitude than IPM for SOCP and QP. We emphasize that
although we present this algorithm for GFlasso, it can also be
applied to solve any regression problems involving a fusion penalty
such as fused lasso with a univariate response, where the fusion
penalty can be defined on any structures such as chain, grid, or
graphs. In addition, it is easy to implement with only a few lines
of MATLAB code.

The rest of the paper is organized as follows. In Section 2, we
briefly review lasso and $\ell_1/\ell_2$-regularized regression for
sparse multi-task learning. In Section 3, we present GFlasso. In
Section 4, we present our proximal-gradient optimization method for GFlasso
and other related optimization problems, and discuss
its convergence rate and complexity analysis. In Section 5, we
present the preliminary consistency result of GFlasso. In Section 6,
we demonstrate the performance of the proposed method on simulated
and asthma datasets, followed by conclusions in Section 7.

\section{Preliminary: $\ell_1$- and $\ell_1/\ell_2$-Reguliarized Multi-task Regression   }

Assume a sample of $N$ instances, each represented by a
$J$-dimensional input vector and a $K$-dimensional output vector.
Let $\mathbf{X}=(\mathbf{x}_1, \ldots, \mathbf{x}_J) \in
\mathbb{R}^{N \times J}$ denote the input matrix, and let
$\mathbf{Y}=(\mathbf{y}_1, \ldots, \mathbf{y}_K) \in \mathbb{R}^{N
\times K}$ represent the output matrix. For each of the $K$ output
variables (so called tasks), we assume a linear model:
\begin{equation}
    \mathbf{y}_k = \mathbf{X}\bm{\beta}_k+\bm{\epsilon}_k,
    \quad \forall k=1,\ldots, K,
    \label{eq:m}
\end{equation}
where $\bm{\beta}_k =(\beta_{1k}, \ldots, \beta_{Jk})^T \in
\mathbb{R}^J$ is the vector of regression coefficients for the
$k$-th output variable, and $\bm{\epsilon}_k$ is a vector of $N$
independent zero-mean Gaussian noise. We center the $\mathbf{y}_k$'s
and $\mathbf{x}_j$'s such that $\sum_{i=1}^N y_{ik}=0$ and
$\sum_{i=1}^N x_{ij}=0$, and consider the model without an
intercept.


Let $\mathbf{B}=(\bm{\beta}_1, \ldots, \bm{\beta}_K)$ denote the
$J\times K$ matrix of regression coefficients of all $K$ response
variables. Then, lasso \cite{Tibshirani:96}  obtains
$\hat{\mathbf{B}}^{\textrm{lasso}}$ by solving the following
optimization problem:
\begin{equation}
    \widehat{\mathbf{B}}^{\textrm{lasso}} = \argmin_{\bB} \frac{1}{2}
    \|\bY-\bX\bB\|_F^2+\lambda \|\bB\|_1,
\label{eq:lasso}
\end{equation}
where $\|\cdot\|_F$ denotes the matrix Frobenius norm, $\|\cdot\|_1$
denotes the the entry-wise matrix $\ell_1$-norm (i.e.,
$\|\bB\|_1=\sum_{k=1}^K\sum_{j=1}^J|\bb_{jk}|$) and $\lambda$ is a
regularization parameter that controls the sparsity level. We note
that lasso in \eqref{eq:lasso} does not offer any mechanism for a
joint estimation of the parameters for the multiple outputs.

Recently, a mixed-norm (e.g., $\ell_1/\ell_2$) regularization has
been used for a recovery of joint sparsity across multiple tasks,
when the tasks share the common set of relevant covariates
\cite{Argyriou:06,Obozinski:09}. More precisely, it encourages the
relevant covariates to be shared across output variables and finds
estimates in which only few covariates have non-zero regression
coefficients for one or more of the $K$ output variables. The
corresponding optimization problem is given as follows:
\begin{equation}
\widehat{\mathbf{B}}^{l_1/l_2} = \argmin_{\bB} \frac{1}{2}
    \|\bY-\bX\bB\|_F^2+\lambda \|\bB\|_{1,2},
\label{eq:l1l2}
\end{equation}
where $\|\bB\|_{1,2}=\sum_{j=1}^J \|\bm{\beta}^j\|_2$,
$\bm{\beta}^j$ is the $j$-th row of the regression coefficient
matrix $\bB$, and $\|\cdot\|_2$ denotes the vector $\ell_2$-norm.
Although the mixed-norm allows information to be combined across
output variables, it assumes all of the tasks are equally related to
inputs, and cannot incorporate a complex structure in how the
outputs themselves are correlated.

\section{Graph-guided Fused Lasso for Sparse Structured Multitask Regression}

In this section, we propose GFlasso that explicitly takes
into account the complex  dependency structure in the output variables
represented as a graph while estimating the regression
coefficients. We assume that the output structure of the $K$
output variables is available as a graph $G$ with a set of nodes
$V=\{1,\ldots,K\}$ and edges $E$. In this paper, we adopt a simple
strategy for constructing such graphs by computing pairwise
correlations based on $\mathbf{y}_k$'s, and connecting two nodes
with an edge if their correlation is above a given threshold $\rho$.
More sophisticated methods can be easily employed, but they are not
the focus of this paper. Let $r_{ml} \in \mathbb{R}$ denote the
weight (can be either positive or negative) of an edge $e=(m,l) \in
E$ that represents the strength of correlation between the two
nodes. Here, we simply adopt the Pearson's correlation between
$\by_m$ and $\by_l$ as $r_{ml}$.

Given the graph $G$, it is reasonable to assume that if two output
variables are connected with an edge in the graph, they tend to be
influenced by the same set of covariates with similar strength. In
addition, we assume that the edge weights in the graph $G$ contain
information on how strongly the two output variables are related
and thus share relevant covariates. GFlasso employs an additional
constraint over the standard lasso by fusing the $\beta_{jm}$ and
$\beta_{jl}$ if $(m,l) \in E$ as follows:
\begin{equation}
  \label{eq:graph_obj}
   \hat{\mathbf{B}}^{\textrm{GF}}=\min_{\bB} f(\bB) \equiv  \frac{1}{2}\|\bY-\bX\bB\|_F^2+\lambda\|\bB\|_1+\gamma\sum_{e=(m,l)\in
   E}\tau(r_{ml})\sum_{j=1}^J|\bb_{jm}-\mbox{sign}(r_{ml})\bb_{jl}|,
\end{equation}
where $\lambda$ and $\gamma$ are regularization parameters that
control the complexity of the model. A larger value for $\gamma$ leads
to a greater fusion effect. In this paper, we consider $\tau(r)=|r|$,
but any positive monotonically increasing function of the absolute
value of correlations can be used. The $\tau(r_{ml})$ weights the
fusion penalty for each edge such that $\beta_{jm}$ and $\beta_{jl}$
for highly correlated outputs with large $|r_{ml}|$ receive a
greater fusion effect than other pairs of outputs with weaker
correlations. The $\mysgn(r_{ml})$ indicates that two negatively
correlated outputs are encouraged to have the same set of relevant
covariates with the same absolute value of regression coefficients
of the opposite sign.

When the edge-level fusion penalty is applied to all of the edges in
the entire graph $G$ in the GFlasso penalty, the overall
effect is that each subset of output variables within a densely
connected subgraph tends to have common relevant covariates. This is
because the fusion effect propagates through the neighboring edges,
fusing the regression coefficients for each pair of outputs
connected by an edge, where the amount of such propagation is
determined by the level of local edge connectivities and edge
weights.

The idea of using a fusion penalty has been first proposed for the problem
with a univariate response and high-dimensional covariates to fuse the
regression coefficients of two adjacent covariates when the
covariates are assumed to be ordered such as in time
\cite{Tibshirani:05}. In GFlasso, we employ a similar but more
general strategy in a multiple-output regression in order to
identify shared relevant covariates for related output variables.

\section{Proximal-Gradient Method for Optimization}
\label{sec:prox_grad}

Although the optimization problem for GFlasso in
\eqref{eq:graph_obj} is convex,
it is not trivial to optimize it because of the non-smooth penalty function.
The problem can be
formulated as either SOCP or QP using the similar strategy in
\cite{Tibshirani:05}. The state-of-the-art approach for solving SOCP
and QP are based on an IPM that requires
solving a Newton system to find a search direction.
Thus, it is computationally very expensive even for problems of
moderate size.

In this section, we propose a proximal-gradient method which
utilizes only the first-order information with a fast convergence
rate and low computation complexity per iteration. More precisely,
we first reformulate the $\ell_1$ and fusion penalty altogether into
a max problem over  the auxiliary variables, and then introduce its
smooth lower bound. Instead of optimizing the original penalty, we
adopt the accelerated gradient descent method \cite{Nesterov:05} to
optimize the smooth lower bound.

The approach of ``proximal'' method is quite general in that it
optimizes a lower or upper bound of the original objective function,
rather than optimize the objective function directly. This lower or
upper bound has a simpler form that allows for an easy optimization.
Recently, different variants of a proximal-gradient method have been
applied to solve optimization problems with a convex loss and
non-smooth penalty, including matrix completion \cite{Shuiwang:09}
and sparse signal reconstruction \cite{Jieping:09}. However, the
non-smooth penalties in these works are the norm of the coefficient
itself instead of the norm of a linear transformation of the
coefficient. So they use a quadratic approximation to the smooth
loss function while keeping the easily-handled non-smooth penalty in
their original form. In contrast, in this work, we find a smooth
lower bound of the complicated non-smooth penalty term.


\subsection{A Reformulation of the Non-smooth Penalty Term}
\label{sec:reform}

First, we rewrite the graph-guided fusion penalty function in
\eqref{eq:graph_obj}, using a vertex-edge incident matrix $H \in
\mathbb{R}^{K \times |E|}$, as follows:
\begin{equation*}
\sum_{e=(m,l)\in
E}\tau(r_{ml})\sum_{j=1}^J|\bb_{jm}-\mbox{sign}(r_{ml})\bb_{jl}|
\equiv \|\bB    H\|_1,
\end{equation*}
where $H \in \mathbb{R}^{K \times |E|}$ is a variant vertex-edge
incident matrix defined as below:
\begin{equation*}
H_{k,e}=\left\{
\begin{array}{ll}
\tau(r_{ml})&\mbox{if }e=(m,l)\mbox{ and }k=m\\
-\mbox{sign}(r_{ml})\tau(r_{ml})&\mbox{if }e=(m,l)\mbox{ and }k=l\\
0&\mbox{otherwise.}
\end{array}
\right.
\end{equation*}
Therefore, the overall penalty in \eqref{eq:graph_obj} including
both lasso and graph-guided fusion penalty functions can be written as $\|\bB
C\|_1$, where $C=(\lambda I, \gamma H)$ and $I \in \mathbb{R}^{K
\times K}$ denotes an identity matrix.

Since the dual norm of the entry-wise matrix $\ell_\infty$ norm is the
$\ell_1$ norm, we can further rewrite the overall penalty as:
\begin{equation}
\label{eq:dual_norm}
  \|\bB C\|_1 \equiv \max_{\|\bA\|_{\infty}\leq1} \langle \bA,\bB C \rangle,
\end{equation}
where $\langle \mathbf{U}, \mathbf{V} \rangle \equiv
\Tr(\mathbf{U}^T\mathbf{V})$ denotes a matrix inner product, $\bA
\in \mathcal{Q}=\{\bA | \|\bA\|_{\infty} \leq 1, \bA \in
\mathbb{R}^{J \times (K+|E|)}\}$ is an auxiliary matrix associated
with $\|\bB C\|_1$, and $\|\cdot\|_\infty$ is the matrix entry-wise
$\ell_\infty$ norm, defined as the maximum absolute value of all
entries in the matrix.

According to \eqref{eq:dual_norm}, the penalty term can be viewed as
the inner product of the auxiliary matrix $\bA$ and the linear
mapping of $\bB$, $\Gamma(\bB) \equiv \bB C$, where the linear
operator $\Gamma$ is a mapping from $\mathbb{R}^{J\times K}$ into
$\mathbb{R}^{J\times (K+|E|)}$. By the fact that $\langle \bA,
\Gamma(\bB) \rangle \equiv \Tr(\bA^T\bB C) = \Tr(C\bA^T\bB) \equiv
\langle \bA C^T, \bB \rangle $, the adjoint operator of $\Gamma$ is
$\Gamma^{\ast}(\bA)=\bA C^T$ that maps $\mathbb{R}^{J\times
(K+|E|)}$ back into $\mathbb{R}^{J\times K}$. Essentially, the
adjoint operator $\Gamma^{\ast}$ is the linear operator induced by
$\Gamma$ defined in the space of auxiliary variables.  The use of
the linear mapping $\Gamma$ and its adjoint $\Gamma^{\ast}$ will
simplify our notation and provide a more consistent way to present
our key theorem as shown in the next section.

\subsection{Proximal-Gradient Method}

The formulation of the penalty in \eqref{eq:dual_norm} is still a
non-smooth function in $\bB$, and this makes the optimization still
challenging. To tackle this problem, we introduce an auxiliary
strongly convex function to construct a smooth approximation of
\eqref{eq:dual_norm}. More precisely, we define:
\begin{equation}
\label{eq:fmu}
  f_\mu(\bB)=\max_{\|\bA\|_\infty \leq 1} \langle \bA,\bB C \rangle - \mu d(\bA),
\end{equation}
where $\mu$ is a positive smoothness parameter and $d(\bA)$ is an
arbitrary smooth strongly-convex function defined on $\mathcal{Q}$.
The original penalty term can be viewed as $f_\mu(\bB)$ with
$\mu=0$, i.e. $f_0(\bB)=\max_{\|\bA\|_\infty \leq 1} \langle \bA,\bB
C \rangle$. Since our algorithm will utilize the optimal solution
$\bA^{\ast}$ to \eqref{eq:fmu},  we choose $d(\bA) \equiv
\frac{1}{2} \|\bA\|_F^2$ so that we can obtain the closed-form
solution for $\bA^{\ast}$.

It can be easily seen that $f_\mu(\bB)$ is a lower bound of $f_0(\bB)$. To
bound the gap between them, let
\begin{equation}
\label{eq:D} D= \max_{\|\bA\|_ \infty \leq 1 } d(\bA)= \frac{1}{2}
\|\bA\|_F^2 =\frac{1}{2}J(K+|E|). \end{equation}
Then, we have $f_0
(\bB)-f_\mu(\bB) \leq \mu D = \mu J(K+|E|)/2$. From the key theorem
we present below, we know that $f_\mu(\bB)$ is a smooth function for
any $\mu>0$. Therefore, $f_\mu(\bB)$ can be viewed as a smooth
approximation of $f_0(\bB)$ with the maximum gap of $\mu D$ and
the $\mu$ controls the gap between $f_\mu(\bB)$ and $f_0(\bB)$. As we
discuss in the next section in more detail, if our desired accuracy
is $\epsilon$, i.e., $f(\bB^t)-f(\bB^{\ast}) \leq \epsilon$, where
$\bB^t$ is the approximate solution at the $t$-th iteration, and
$\bB^{\ast}$ is the optimal solution to the objective function in
\eqref{eq:graph_obj}, we should set $\mu=\frac{\epsilon}{2D}$ to
achieve the best convergence rate.

Now, we present the key theorem to show that $f_\mu(\bB)$ is smooth
and that $\nabla f_\mu(\bB)$ is Lipschitz continuous. This theorem is also stated
in \cite{Nesterov:05} but without a detailed proof of the smoothness
property and a derivation of the gradient. We provide a simple
proof based on the Fenchel Conjugate and properties of
subdifferential. The details of the proof are presented in
Appendix.
\begin{theorem}
For any $\mu>0$,  $f_\mu(\bB)$ is a convex and continuously
differentiable function in $\bB$ with the gradient:
\begin{equation}
  \label{eq:grad}
  \nabla f_\mu(\bB)= \Gamma^{\ast}(\bA^{\ast}) =\bA^{\ast}  C^T,
\end{equation}
where $\Gamma^{\ast}$ is the adjoint operator of $\Gamma$ defined at
the end of Section \ref{sec:reform}; $\bA^{\ast}$ is the optimal
solution to \eqref{eq:fmu}. Furthermore, the gradient $\nabla
f_\mu(\bB)$ is Lipschitz continuous with the Lipschitz constant
$L_\mu=\frac{1}{\mu} \|\Gamma\|^2$, where $\|\Gamma\|$ is the norm
of the linear operator $\Gamma$ defined as $\|\Gamma\| \equiv
\max_{\|\bV\|_2 \leq 1} \|\Gamma(\bV)\|_2$.
 \label{thm:key}
\end{theorem}

To compute the $\nabla f_\mu(\bB)$ and $L_\mu$ in the above theorem, we need to know
$\bA^{\ast}$ and $\|\Gamma\|$. We present the closed-form expressions of
$\bA^{\ast}$ and $\|\Gamma\|$ in the following two lemmas. The proof
of Lemma \ref{lem:norm} is provided in Appendix.
\begin{lemma}
Let $\bA^{\ast}$ be the optimal solution of \eqref{eq:fmu}:
\begin{equation*}
 \bA^{\ast}= S(\frac{\bB C}{\mu}),
\end{equation*}
where $S$ is the shrinkage operator defined as follows.  For $x \in \mathbb{R}$,
$S(x)=x$ if $-1< x < 1$,
$S(x)=1$ if $x \geq 1$, and $S(x)=-1$ if $x \leq -1$.
For matrix $\bA$, $S(\bA)$ is defined as applying $S$
on each and every entry of $\bA$.
\end{lemma}
\begin{proof}
By taking the derivative of \eqref{eq:fmu} over $\bA$ and setting it to
zeros, we obtain $ \bA=\frac{\bB C}{\mu}$. Then, we project this
solution onto the $\mathcal{Q}$ and get the optimal solution $\bA^{\ast}$.
\end{proof}

\begin{lemma}
$\|\Gamma\|$ is upper bounded by $\|\Gamma\|_U \equiv \sqrt{
\lambda^2+2\gamma^2\max_{k\in V}d_k}$, where
\begin{equation}
\label{eq:d} d_k=\sum_{e\in E \ s.t. \  e \ \mathrm{incident\  on} \
k}(\tau(r_e))^2
\end{equation}
 for $k\in V$ in graph $G$; and this bound is tight.
\label{lem:norm}
\end{lemma}

Given the results in Theorem \ref{thm:key}, now we present our
proximal-gradient method for GFlasso.
We substitute the penalty term in \eqref{eq:dual_norm} with its
smooth approximation $f_\mu(\bB)$ and obtain the smooth optimization
problem:
\begin{equation}
\label{eq:approx_glasso}
  \min_{\bB} \tilde{f}(\bB) \equiv \frac{1}{2} \|\bY-\bX\bB\|_2^2 +  f_\mu(\bB).
\end{equation}
According to Theorem \ref{thm:key}, the gradient of $\tilde{f}(\bB)$
is:
\begin{equation}
\label{eq:fgrad}
  \nabla \tilde{f}(\bB)= \bX^T(\bX \bB - \bY)  + \bA^{\ast}C^T.
\end{equation}
Note that $\nabla \tilde{f}(\bB)$ is Lipschitz continuous with the
Lipschitz constant  $L$ tightly upper bounded by $L_U$:
\begin{equation}
\label{eq:L}
  L=\lambda_{\max} (\bX^T\bX)+ L_{\mu} \leq \lambda_{\max}
  (\bX^T\bX) +\frac{\lambda^2+2\gamma^2\max_{k\in
  V}d_k}{\mu} \equiv L_U,
\end{equation}
where $\lambda_{\max} (\bX^T\bX)$ is the largest eigenvalue of
$(\bX^T\bX)$.

Instead of optimizing the original function $f(\bB)$ in
\eqref{eq:graph_obj}, we optimize $\tilde{f}(\bB)$. Since
$\tilde{f}(\bB)$ is a \emph{smooth} lower bound of $f(\bB)$, we can
adopt the accelerated gradient-descent method, so called Nesterov's
method \cite{Nesterov:05}, to minimize smooth $\tilde{f}(\bB)$ as
shown in Algorithm \ref{algo:gdglasso}.

\begin{algorithm}[!th]
\caption{Proximal-Gradient Method for GFlasso} \textbf{Input}:
$\bX$, $\bY$, $\lambda$, $\gamma$, graph structure $G$, desired
accuracy $\epsilon$.

\textbf{Initialization}:  Construct $C=(\lambda I, \gamma H)$;
compute $L_U$ according to \eqref{eq:L}; compute  $D$ in
\eqref{eq:D} and set $\mu=\frac{\epsilon}{2D}$; set $\bW^0
=\mathbf{0}\in \mathbb{R}^{J \times K}$;

\textbf{Iterate} For $t=0,1,2,\ldots$ until convergence of $\bB^t$:
\begin{enumerate}
\item Compute $\nabla \tilde{f}(\bW^t)$ according to
\eqref{eq:fgrad}.
\item Perform the gradient descent step :
$\bB^t= \bW^t - \frac{1}{L_U} \nabla \tilde{f}(\bW^t)$.
\item Set $ \bZ^t= -\frac{1}{L_U} \sum_{i=0}^t \frac{i+1}{2} \nabla
\tilde{f}(\bW^i)$.
\item Set $ \bW^{t+1} = \frac{t+1}{t+3}  \bB^t + \frac{2}{t+3} \bZ^t$.
\end{enumerate}
\textbf{Output}: $\hat{\bB} =\bB^t$ \label{algo:gdglasso}
\end{algorithm}

In contrast to the standard gradient-descent algorithm, Algorithm
\ref{algo:gdglasso} involves the updating of three sequences
$\{\bW^t\}$, $\{\bB^t\}$ and $\{\bZ^t\}$, where $\bB^t$ is obtained
from the gradient-descent update based on $\bW^t$ with the stepsize
$\frac{1}{L_U}$; $\bZ^t$ is the weighted combination of all previous
gradient information and $\bW^{t+1}$ is the convex combination of
$\bB^t$ and $\bZ^t$. Intuitively, the reason why this method is
superior to the standard gradient descent is that it utilizes all of
the gradient information from the first step to the current one for
each update, while the standard gradient-descent update is only
based on the gradient information at the current step.

\subsection{Complexity}
\label{subsec:complexity}

Although we optimize the approximation function $\tilde{f}$, it
still can be proven that the $\widehat{\bB}$ obtained from Algorithm
\ref{algo:gdglasso} is sufficiently close to the  optimal solution
$\bB^{\ast}$ to the original objective function in
\eqref{eq:graph_obj}. We present the convergence rate of Algorithm
\ref{algo:gdglasso} in the next theorem.

\begin{theorem}
\label{thm:complexity} Let $\bB^{\ast}$ be the optimal solution to
\eqref{eq:graph_obj} and $\bB^t$ be the intermediate solution at
the $t$-th iteration in Algorithm \ref{algo:gdglasso}. If we require
$f(\bB^t)-f(\bB^{\ast}) \leq \epsilon$ and set
$\mu=\frac{\epsilon}{2D}$, then the number of iterations $t$ is
upper bounded by:
\begin{equation}
\label{eq:bound}
 \sqrt{\frac{4
 \|\bB^{\ast}\|_F^2}{\epsilon}{\left(\lambda_{\max}(\bX^T\bX)+\frac{2D\|\Gamma\|_U^2}{\epsilon}\right)}},
\end{equation}
where $D$ and $\|\Gamma\|$  are as in \eqref{eq:D} and Lemma
\ref{lem:norm} respectively.
\end{theorem}
The key idea behind the proof is to decompose
$f(\bB^t)-f(\bB^{\ast})$ into 3 parts: (i)
$f(\bB^t)-\tilde{f}(\bB^t)$, (ii)
$\tilde{f}(\bB^t)-\tilde{f}(\bB^{\ast})$, (iii)
$\tilde{f}(\bB^{\ast})-f(\bB^{\ast})$.  (i) and (iii) can be bounded
by the gap of the approximation $\mu D$. Since $\tilde{f}$ is a
smooth function, we can bound (ii) by the accuracy bound when
applying the accelerated gradient method to minimize smooth
functions \cite{Nesterov:05}. We obtain \eqref{eq:bound} by
balancing these three terms. The details of the proof are presented
in Appendix. According to Theorem \ref{thm:complexity}, Algorithm
\ref{algo:gdglasso} converges in $O(\frac{\sqrt{2D}}{\epsilon})$
iterations, which is much faster than the subgradient method with
the convergence rate of $O(\frac{1}{\epsilon^2})$. Note that the
convergence rate of our method depends on $D$ through the term
$\sqrt{2D}$, which again depends on the problem size with
$D=J(K+|E|)/2$.

Theorem \ref{thm:complexity} suggests that a good strategy for
choosing the parameter $\mu$ in Algorithm \ref{algo:gdglasso} is to
set $\mu=\frac{\epsilon}{2D}$, where $D$ is determined by the
problem size. Instead of fixing the value for $\epsilon$, we
directly set $\mu$ or the ratio of $\epsilon$ and $D$ to a constant,
because this automatically has the effect of scaling $\epsilon$
according to the problem size without affecting the quality of the
solution.

Assuming that we pre-compute and store $\bX^T\bX$ and $\bX^T\bY$
with the time complexity of $O(J^2N+JKN)$, the main computational
cost is to calculate the gradient $\nabla \tilde{f}(\bW_t)$ with the
time complexity of $O(J^2K+ J|E|)$ in each iteration. Note that the
per-iteration complexity of our method is (i) independent of sample
size $N$, which can be very large for large-scale applications, and
(ii) linear in the number of edges $|E|$, which can also be large in many cases. 
In comparison, the second-order method such as SOCP has a much
higher complexity per iteration. According to \cite{SOCP:98}, SOCP
costs $O(J^2(K+|E|)^2(KN+JK+J|E|))$ per iteration, thus cubic in the
number of edges and linear in sample size. Moreover, each IPM
iteration of SOCP requires significantly more memory to store the
Newton linear system.

\subsection{Proximal-Gradient Method for General Fused Lasso }

The proximal-gradient method for GFlasso that we presented in the
previous section can be easily adopted to efficiently solve any
types of optimization problems with a smooth convex loss function
and fusion penalties such as fused-lasso regression and fused-lasso
signal approximator \cite{Tibshirani:05,Holger:09}. We emphasize
that our method can be applied to the fusion penalty defined on an
arbitrary graph structure, while the widely adopted pathwise
coordinate method is only known to be applied to the fusion penalty
defined on special graph structures, i.e., chain and two-way grid.
For example, the general fused lasso solves a univariate regression
problem with a graph fusion penalty as follows:
\begin{equation}
\label{eq:fused}
  \widehat{\bb}=\argmin_{\bb} \frac{1}{2} \|\by-\bX\bb\|_2^2
  +\lambda \sum_{j=1}^J |\beta_j| +\gamma \sum_{e=(m,l)\in
   E}|\beta_m-\beta_l|,
\end{equation}
where $\|\cdot\|_2$ denotes a vector $\ell_2$-norm, $\by \in
\mathbb{R}^N$ is the univariate response vector of length $N$, $\bX
\in \mathbb{R}^{N \times J}$ is the input matrix, $\bb \in
\mathbb{R}^J$ is the regression coefficient vector, and $E$ is the
edge set in graph $G=(V,E)$ with $V=\{1,\ldots,J\}$. Note that the
fusion penalty defined on inputs ordered in time as a chain (i.e.,
$\sum_{j=1}^{J-1} |\beta_{j+1}-\beta_j|$) \cite{Tibshirani:05} is a
special case of the penalty in \eqref{eq:fused}. It is
straightforward to apply our proximal-gradient method to solve
\eqref{eq:fused} with only a slight modification of the linear
mapping $\Gamma(\bb) \equiv C\bb$ and its adjoint
$\Gamma^{\ast}(\ba) \equiv C^T \ba$.


To the best of our knowledge, the only gradient-based method
available for optimizing \eqref{eq:fused} is the recent work
in \cite{Holger:09} that adopts a path algorithm. However, this
relatively complex algorithm works only for the case of $N>J$, and
does not have any theoretical guarantees on convergence. In contrast, our
method is more generally applicable, is simpler to implement,
and has a faster convergence rate.

\section{Asymptotic Consistency Analysis}

It is possible to derive results on the asymptotic behavior of the
GFlasso estimator that is analogous to the ones in lasso
\cite{Tibshirani:96} and fused lasso \cite{Tibshirani:05} when
$J$ and $K$ are fixed as $N \to \infty$. Assume $\bB$ is the true
coefficient matrix and $\widehat{\bB}_N$ is the estimator obtained
by optimizing \eqref{eq:graph_obj}. Also assume that the
regularization parameter $\lambda_N$ and $\gamma_N$ are functions of
$N$. We have the following statistical convergence result:

\begin{theorem}
If $\lambda_N/\sqrt{N} \to \lambda_0 \ge 0$, $\gamma_N/\sqrt{N} \to
\gamma_0 \ge 0$ and $C = \lim_{N \to \infty} \left( \frac{1}{N}
\sum_{i=1}^{N} \mathbf{x}_i\mathbf{x}_i^T\right)$ is non-singular,
where
$\mathbf{x}_i$ is the $i$-th row of $\mathbf{X}$, 
then
    \begin{eqnarray}
    \sqrt{N}(\hat{\mathbf{B}}_N - \mathbf{B}) \to_{\substack{d}}
    \arg\!\min_{\mathbf{U}} V(\mathbf{U}),
    \end{eqnarray}
where
    \begin{eqnarray}
    V(\mathbf{U}) = -2\sum_k \mathbf{u}_k^T\mathbf{W}+\sum_k \mathbf{u}_k^TC\mathbf{u}_k +
        \lambda_0^{(1)} \sum_k \sum_j \big[u_{jk}
        \textrm{sign}(\beta_{jk})I(\beta_{jk} \neq 0)+|u_{jk}|I(\beta_{jk}=0)\big]
        \nonumber \\
        \quad+ \lambda_0^{(2)} \sum_{e=(m,l)\in E} \tau(r_{ml}) \sum_j
        \big[ u_{je}' \mysgn( \beta_{je}')
        I(\beta_{je}' \neq 0 )
        +|u_{je}'|I(\beta_{je}'=0)\big]
        \nonumber
    \end{eqnarray}
with $u_{je}'=u_{jm}-\mysgn(r_{m,l}) u_{jl}$ and
$\beta_{je}'=\beta_{jm}-\mysgn(r_{ml}) \beta_{jl}$, and $\mathbf{W}$
has an $N(\mathbf{0},\sigma^2 C)$ distribution.
\label{thm:statconvergence}
\end{theorem}
The proof of the theorem is provided in Appendix.

\section{Experiments}

In this section, we demonstrate the performance of GFLasso on both
simulated and real data, and show the superiority of our
proximal-gradient method (Prox-Grad) to the existing optimization
methods.
Based on our experience for a range of values for $\mu$, we use $\mu=10^{-4}$
in all of our experiments, since it provided us reasonably good
approximation accuracies across problems of different scales. The
regularization parameters $\lambda$ and $\gamma$ are chosen by
cross-validation. The code is
written in MATLAB and we terminate our optimization procedure when
the relative changes in the objective is below $10^{-6}$.

\subsection{Simulation Study}

\subsubsection{Performance of GFlasso on the Recovery of Sparsity}

We conduct a simulation study to evaluate the performance of the
proposed GFlasso, and compare the results with those from lasso and
\lmult-regularized multi-task regression.

We simulate data using the following scenario analogous to genetic
association mapping with $K = 10$, $J = 30$ and $N = 100$. To
simulate the input data, we use the genotypes of the 60 individuals
from the parents of the HapMap CEU panel, and generate genotypes for
additional 40 individuals by randomly mating the original 60
individuals. We generate the regression coefficients
$\bm{\beta}_k$'s such that the output $\mathbf{y}_k$'s are
correlated with a block-like structure in the correlation matrix. We
first choose input-output pairs with non-zero regression
coefficients as we describe below. We assume three groups of
correlated output variables of sizes 3, 3, and 4. Three relevant
inputs are randomly selected for the first group of outputs, and
four relevant inputs are selected for each of the other two groups,
so that the shared relevant inputs induce correlation among the
outputs within each cluster. In addition, we assume another relevant
input for outputs in both of the first two clusters in order to
model the situation of a higher-level correlation structure across
two subgraphs. Finally, we assume one additional relevant input for
all of the phenotypes. Given the sparsity pattern of $\bB$, we set
all non-zero $\beta_{i,j}$ to a constant $b$ to construct the true
coefficient matrix $\bB$.  Then, we simulate output data based on
the linear-regression model with noise distributed as $N(0,1)$ as in
\eqref{eq:m}, using the simulated genotypes as covariates.

We select the values of the regularization parameters $\lambda$ and $\gamma$ by
using $(N-30)$ samples out of the total $N$ samples as a training
set, and the remaining 30 samples as a validation set. Then,
we use the entire dataset
of size $N$ to estimate the final regression coefficients given the
selected regularization parameters.

\begin{figure}[!t]\centering
\begin{tabular}{@{}c@{}c@{}c@{}c}
\parbox[l]{1.2cm}{
\vspace{-3pt}
\includegraphics[scale = 0.77]{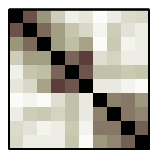}} &
\parbox[l]{1.2cm}{
\vspace{-3pt}
\includegraphics[scale = 0.77]{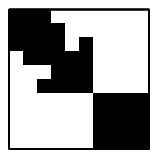}} &
\parbox[l]{5.6cm}{
\includegraphics[scale = 0.77, angle=90]{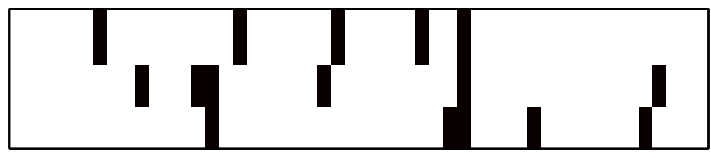}} &
\parbox[l]{5.6cm}{
\includegraphics[scale = 0.77, angle=90]{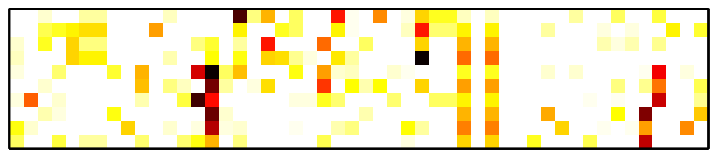}} \\
(a) & (b) & (c)  & (d) \\
& &
\parbox[l]{5.6cm}{
\includegraphics[scale = 0.77, angle=90]{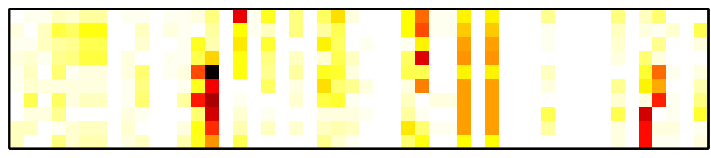}} &
\parbox[l]{5.6cm}{
\includegraphics[scale = 0.77, angle=90]{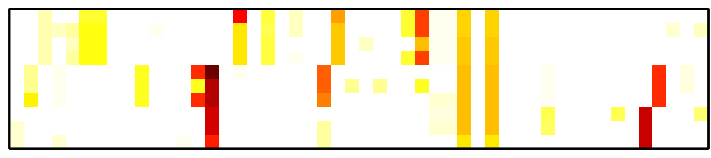}} \\
& & (e) & (f)  \\
\end{tabular}
\caption{Regression coefficients estimated by different methods
based on a single simulated dataset. $b= 0.8$ and threshold
$\rho=0.3$ for the output correlation graph are used. Red pixels
indicate large values. (a) The correlation coefficient matrix of
phenotypes, (b) the edges of the phenotype correlation graph
obtained at threshold 0.3 are shown as white pixels, (c) the true
regression coefficients used in simulation. Absolute values of the
estimated regression coefficients are shown for (d) lasso, (e)
\lmult-regularized multi-task regression, (f) GFlasso. Rows
correspond to outputs and columns to inputs. } \label{fig:sim_b_img}
\end{figure}
\begin{figure}[!t]\centering
\begin{tabular}{cccc}
\includegraphics[scale = 0.4]{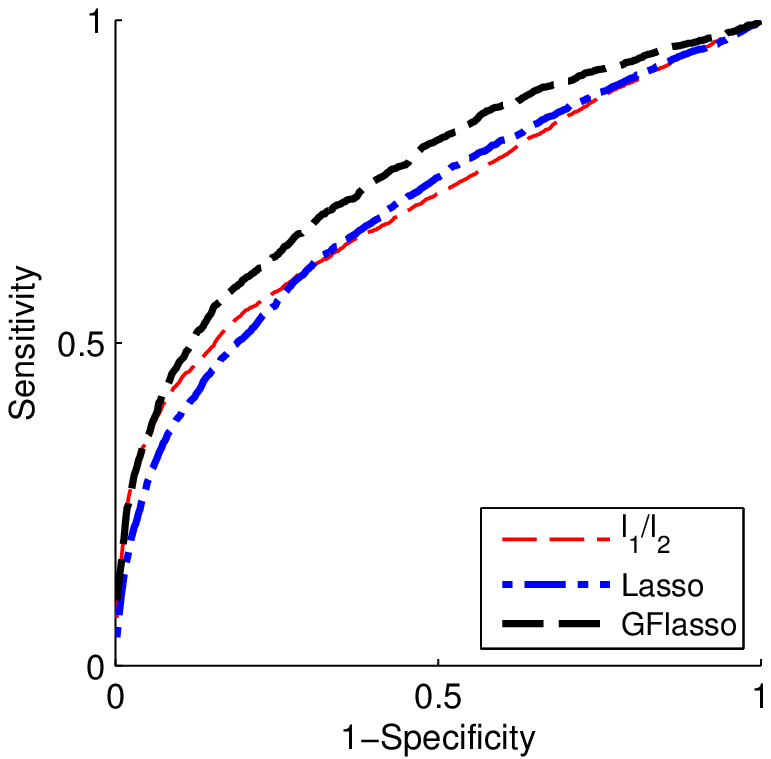} &
\includegraphics[scale = 0.4]{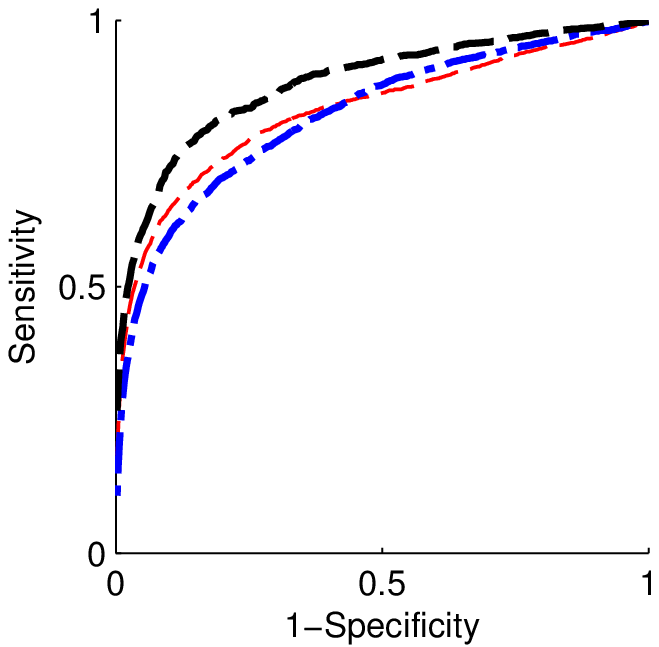} &
\includegraphics[scale = 0.4]{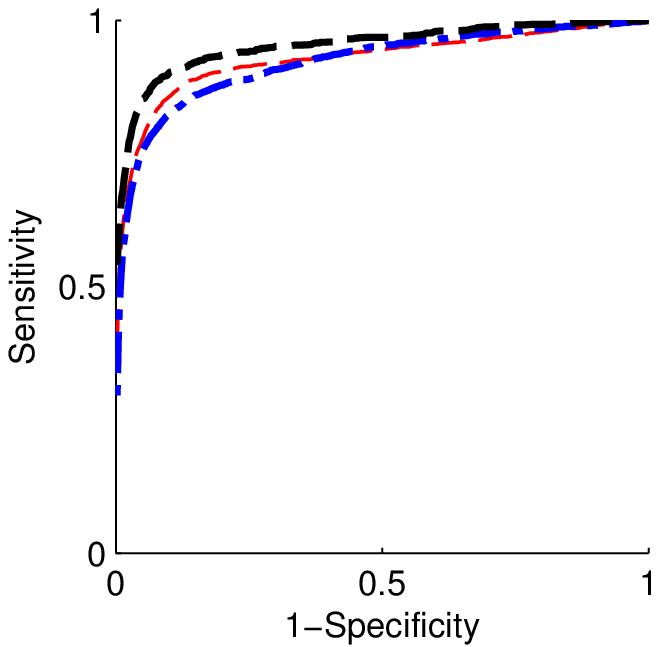} &
\includegraphics[scale = 0.4]{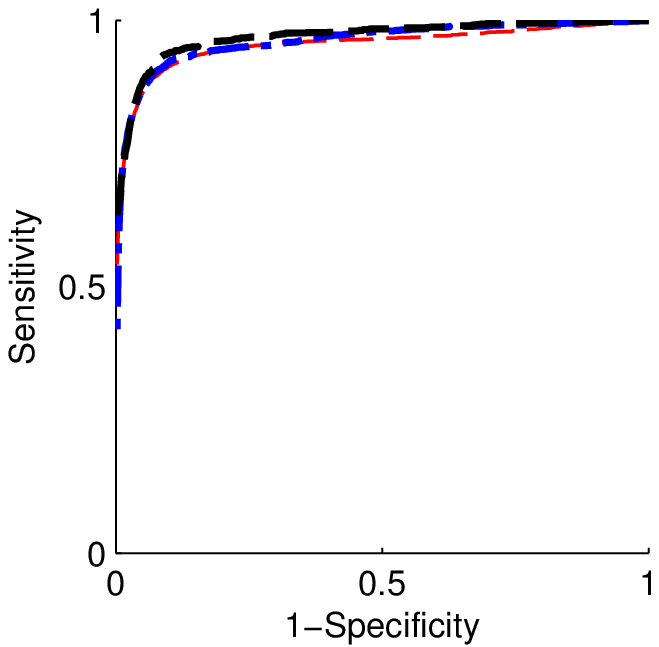} \\
(a) & (b) & (c) & (d)
\end{tabular}
\caption{ROC curves for comparing different sparse regression
methods with varying signal-to-noise ratios. The $b$ is set to (a) 0.3, (b) 0.5,
(c) 0.8, and (d) 1.0. } \label{fig:sim_b}
\end{figure}
\begin{figure}[!t]\centering
\begin{tabular}{cccc}
\includegraphics[scale = 0.4]{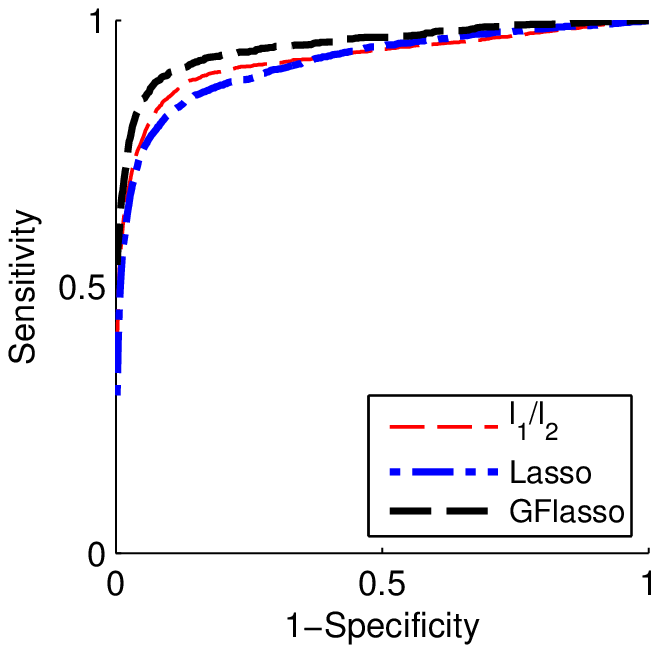} &
\includegraphics[scale = 0.4]{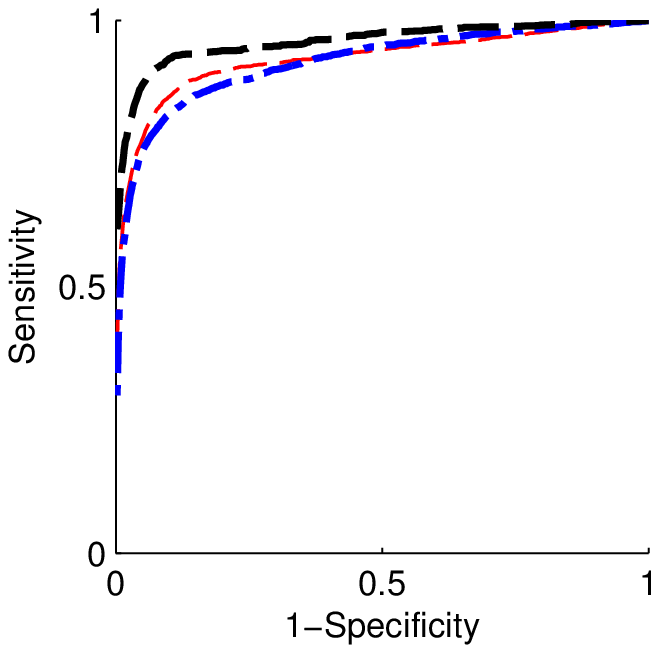} &
\includegraphics[scale = 0.4]{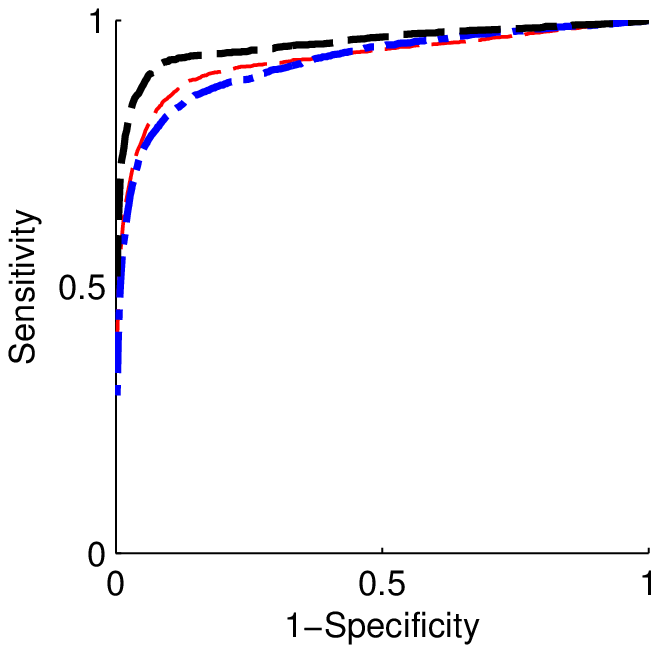} &
\includegraphics[scale = 0.4]{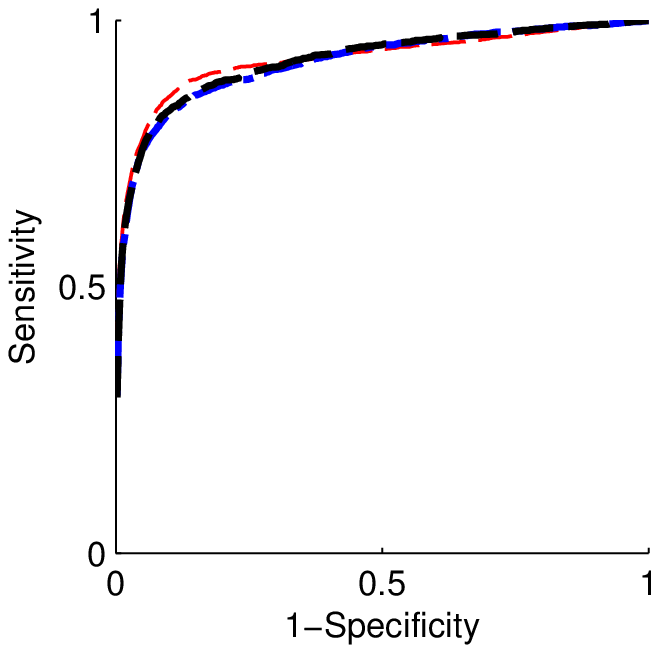} \\
(a) & (b) & (c) & (d)
\end{tabular}
\caption{ ROC curves for comparison of sparse regression methods
with varying thresholds ($\rho$'s) for output graph structures. (a)
$\rho$=0.1, (b) $\rho$=0.3, (c) $\rho$=0.5, and (d) $\rho$=0.7.
We use $b=0.8$ for signal-to-noise ratio.} \label{fig:sim_th}
\end{figure}
\begin{figure}[!t]\centering
\begin{tabular}{cccc}
\includegraphics[scale = 0.4]{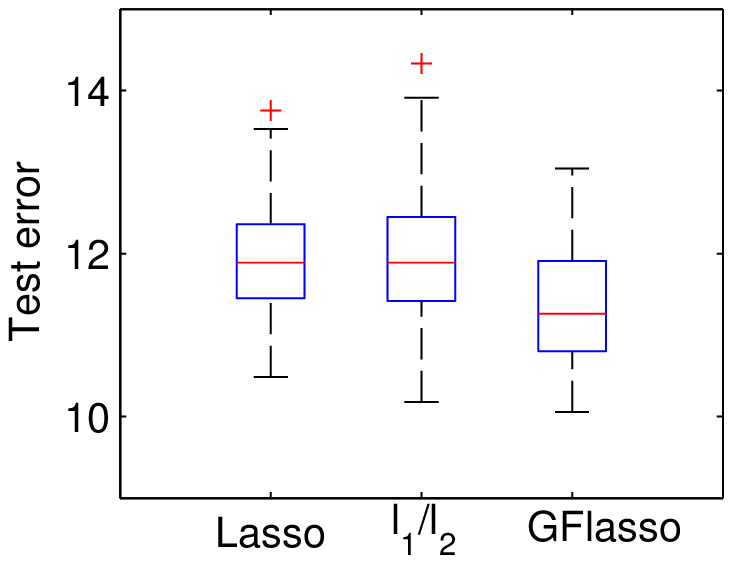} &
\includegraphics[scale = 0.4]{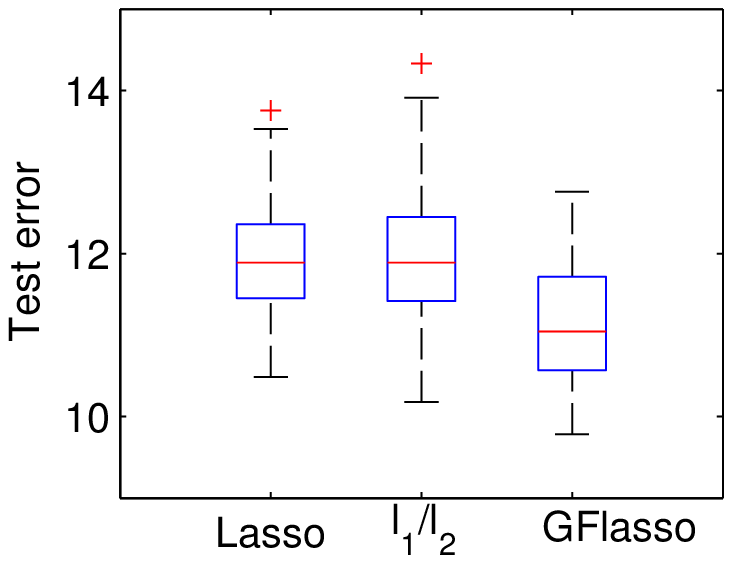} &
\includegraphics[scale = 0.4]{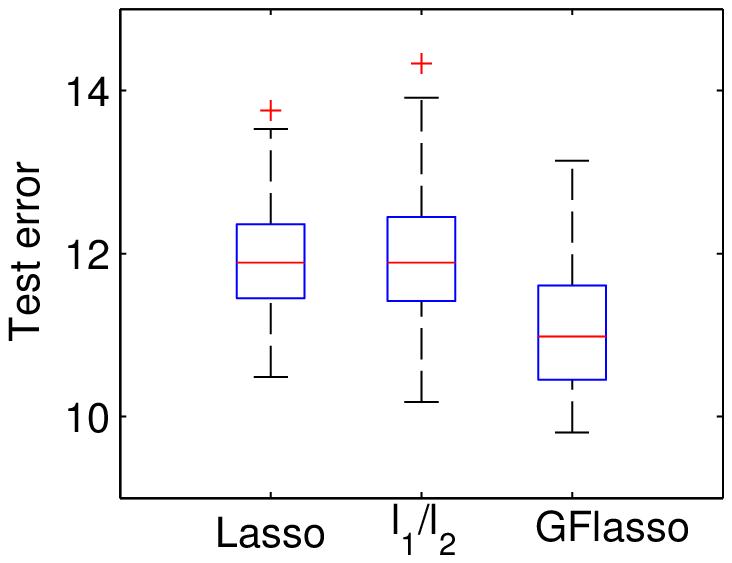} &
\includegraphics[scale = 0.4]{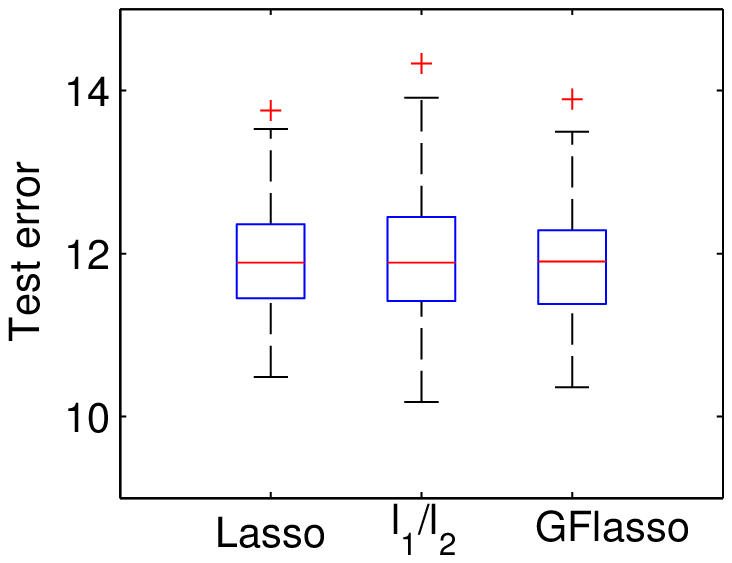} \\
(a) & (b) & (c) 
\end{tabular}
\caption{ Comparison of multi-task learning methods in terms of
prediction error. The threshold $\rho$ for the output correlation
graph is (a) $\rho$=0.1, (b) $\rho$=0.3, (c) $\rho$=0.5, and (d)
$\rho$=0.7.
We use $b=0.8$ for signal-to-noise ratio.
} \label{fig:sim_ts_err}
\end{figure}

As an illustrative example,
a graphical display of the estimated regression
coefficients from different methods
is shown in Figure \ref{fig:sim_b_img}. It is apparent
from Figures \ref{fig:sim_b_img}(d) and (e) that many false
positives show up in the results of lasso and \lmult-regularized
multi-task regression. On the other hand, the results from GFlasso
in Figure \ref{fig:sim_b_img}(f) show fewer false positives and
reveal clear block structures. This experiment suggests that
borrowing information across correlated outputs in the output
graph as in GFlasso can significantly increase the
power of discovering true relevant inputs.

We systematically and quantitatively evaluate the performance
of the various methods by computing sensitivity/specificity on the recovered
sets of relevant inputs and prediction errors averaged over 50 randomly
generated datasets.
We generate additional 50 individuals in each training dataset, and compute
the prediction error for this test dataset.
In order to examine how varying the signal-to-noise ratio affects
the performances of the different methods, we simulate datasets with
the non-zero elements of the regression coefficient matrix $\bB$ set
to $b=0.3$, $0.5$, $0.8$, and $1.0$, and compute the ROC curves as
shown in Figure \ref{fig:sim_b}. A threshold of $\rho$=0.1 is used
to generate output correlation graphs in GFlasso. We find that
GFlasso outperforms the other methods for all of the four chosen
signal-to-noise ratios.

Next, we examine the sensitivity of GFlasso to how the output
correlation graph is generated, by varying the threshold $\rho$ of
edge weights from 0.1 to 0.3, 0.5 and 0.7. With lower values of
$\rho$, more edges would be included in the graph, some of which
represent only weak correlations. The purpose of this experiment is
to see whether the performance of GFlasso is negatively affected by
the presence of these weak and possibly spurious edges that are
included due to noise rather than from a true correlation. The
results are presented in Figure \ref{fig:sim_th}. GFlasso exhibits a
greater power than all other methods even at a low threshold
$\rho$=0.1. As the threshold $\rho$ increases, the inferred graph strucutre
includes only those edges with significant correlations. When the
threshold becomes even higher, e.g., $\rho=0.7$, the number of edges
in the graph becomes close to 0, effectively removing the fusion
penalty. As a result, the performances of GFlasso approaches that of
lasso, and the two ROC curves almost entirely overlap (Figure
\ref{fig:sim_th}(d)). Overall, we conclude that when flexible
structured methods such as GFlasso are used, taking into account the
correlation structure in outputs improves the power of detecting
true relevant inputs regardless of the values for $\rho$. In addition,
once the  graph contains edges that capture strong correlations,
including more edges beyond this point by further lowering the
threshold $\rho$ does not significantly affect the performance of
GFlasso.

Figure \ref{fig:sim_ts_err} shows the prediction errors using
the models learned from the above experiments summarized in Figure
\ref{fig:sim_th}. It can be seen that GFlasso generally offers a
better predictive power than other methods, except for the case
where the set of edges for the graph becomes nearly empty due to
the high correlation threshold $\rho$=0.7 (Figure
\ref{fig:sim_ts_err}(d)). In this case, all of methods
perform similarly.

\subsubsection{Computation Time}

In this section, we compare the computation time of our Prox-Grad
with those of SOCP and QP formulations for solving GFlasso using
simulation data. We use SDPT3 package \cite{SDPT3} to solve the SOCP
formulation. For the QP formulation, we compare two packages, MOSEK
\cite{MOSEK} and CPLEX \cite{CPLEX}, and choose CPLEX since it
performes better in terms of computation time.  The computation time
is reported as the CPU time for one run  on the entire training set
using the best selected regularization parameters.

%

To compare the scalability of Prox-Grad with those of SOCP and QP,
we vary $J$, $N$, $K$, $\rho$ and present the computation time in
seconds in \emph{log-scale} in Figures \ref{fig:syn_scale}(a)-(d),
respectively. All of the experiments are performed on a PC with
Intel Core 2 Quad Q6600 CPU 2.4GHz CPU and 4GB RAM. We point out
that when we vary the threshold $\rho$ for generating the output
graph in Figure \ref{fig:syn_scale}(d), the increase of $\rho$
decreases the number of edges $|E|$ and hence reduces the
computation time. In Figure \ref{fig:syn_scale}, for large values of
$J$, $N$, $K$ and small values of $\rho$, we are unable to collect
results for SOCP and QP, because they lead to out-of-memory errors
due to the large storage requirement for solving the Newton linear
system.

In Figure \ref{fig:syn_scale}, we find that Prox-Grad is
substantially more efficient and can scale up to very
high-dimensional and large-scale datasets. QP is more efficient than
SOCP since it removes the non-smooth $\ell_1$ terms  by introducing
auxiliary variables for each $\ell_1$ term.
In addition,
we notice that the increase of $N$ does not increase the computation
time significantly. This is because $N$ only affects the computation time of
$\bX^T\bX$ and $\bX^T\by$, which can be pre-computed, and does not
affect the time complexity for each iteration during optimization.
This observation is consistent with our complexity analysis in
Section \ref{subsec:complexity}.

\begin{figure}[!th] \centering
  \begin{small}
  \begin{tabular}{cccc}
   \hspace{-0.4cm}  \includegraphics[height=3cm,width=3.5cm]{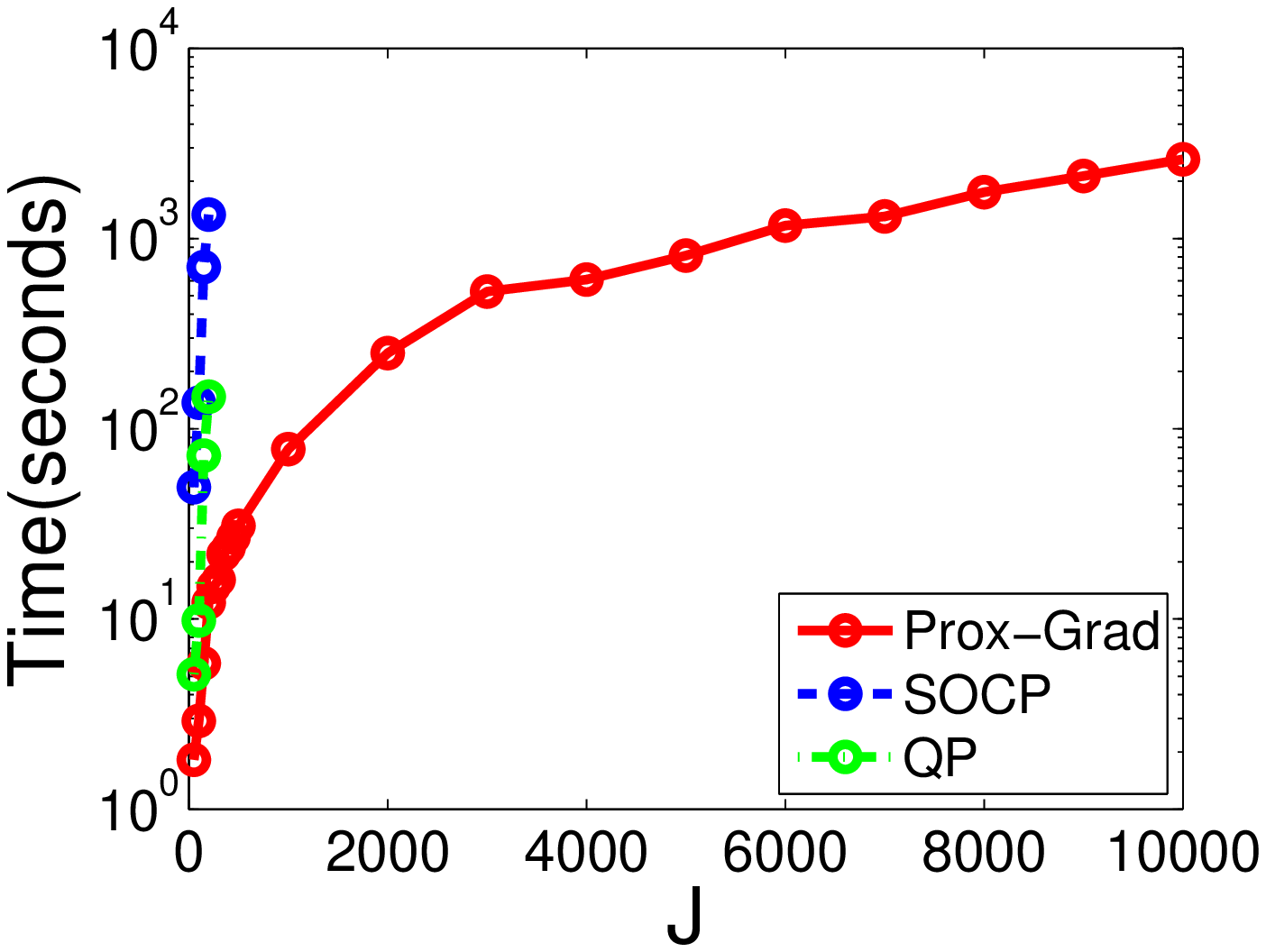}
    &    \hspace{-0.4cm} \includegraphics[height=3cm,width=3.5cm]{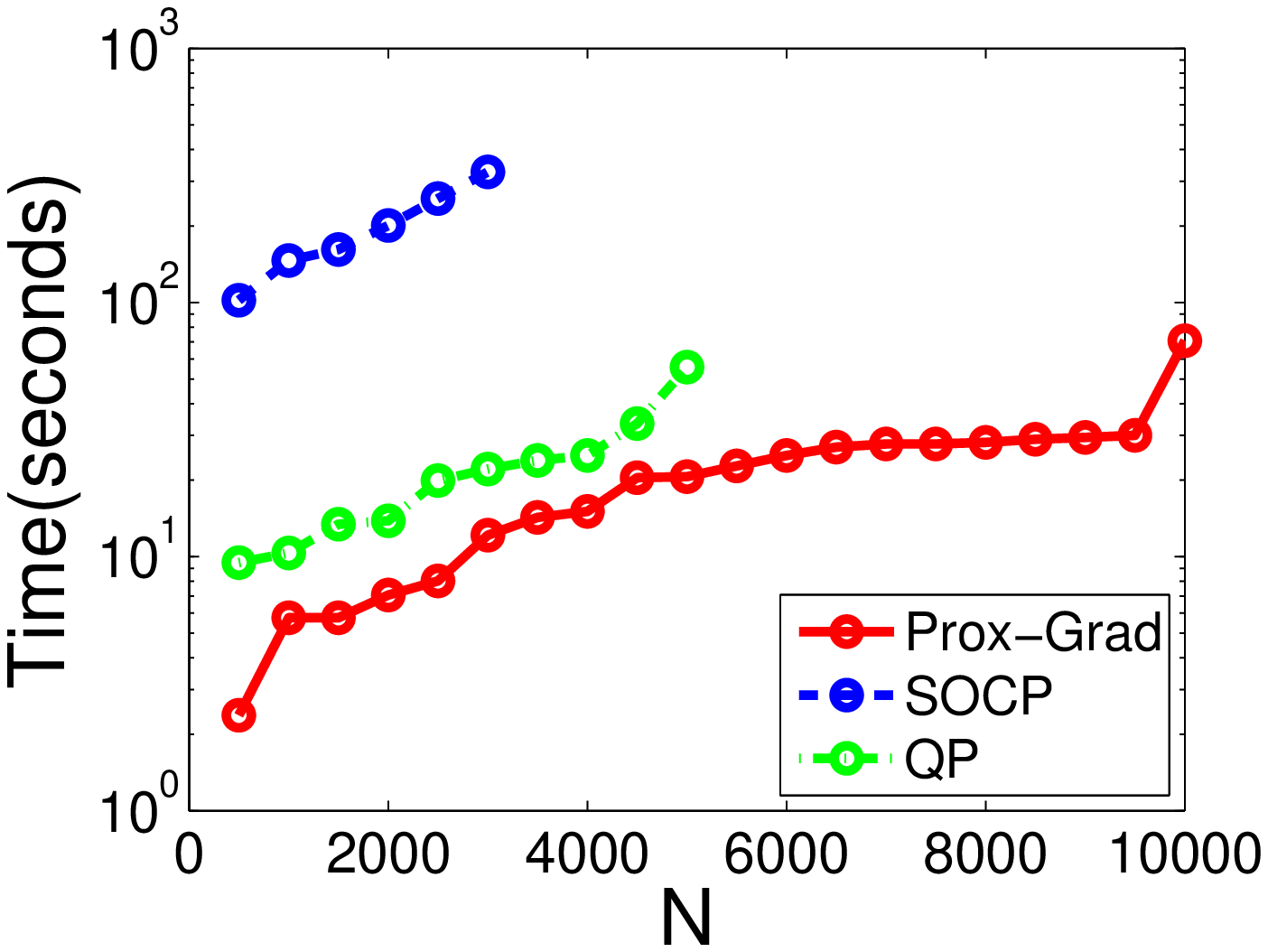}
    &    \hspace{-0.4cm} \includegraphics[height=3cm,width=3.5cm]{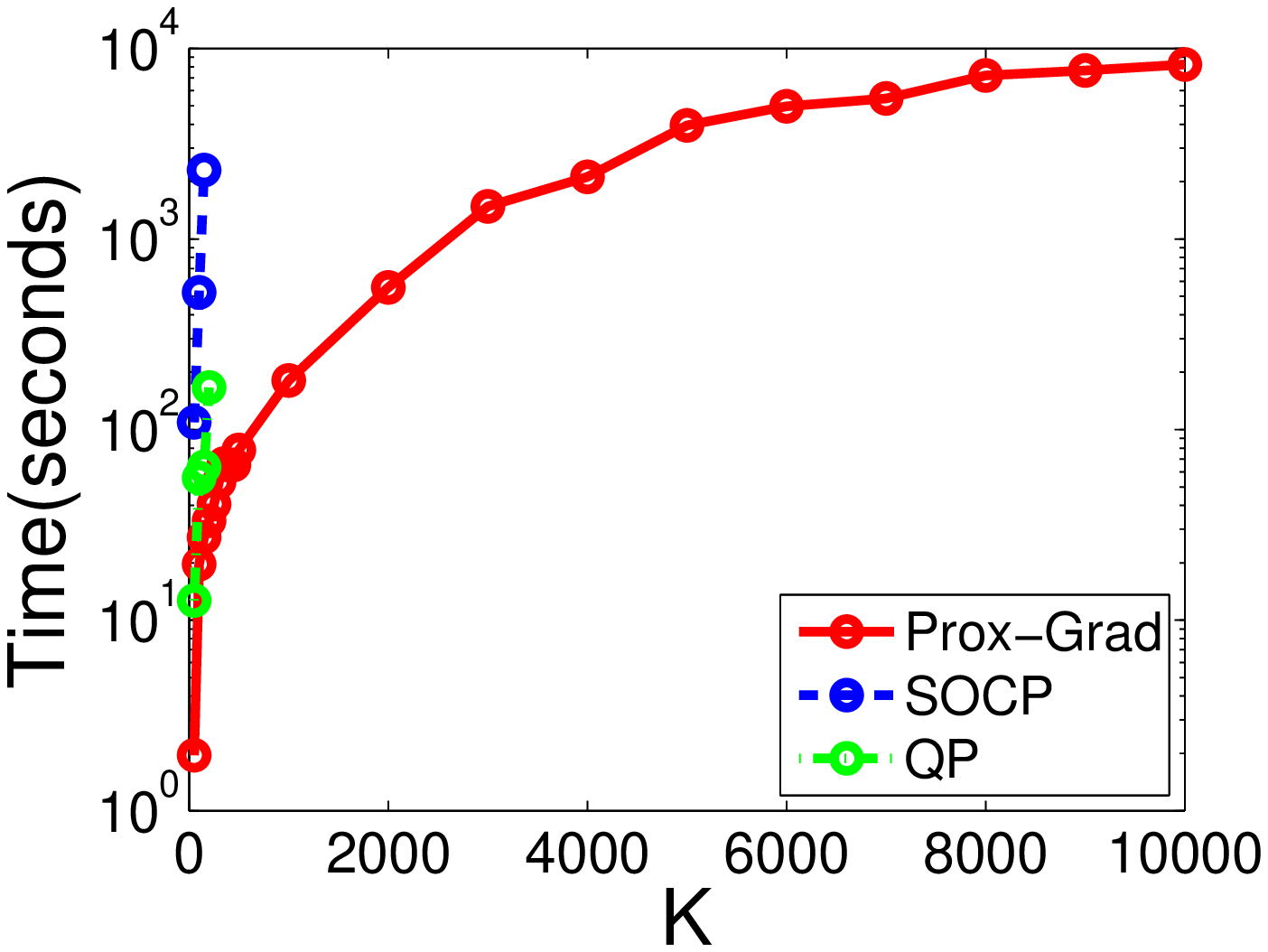} &
       \hspace{-0.4cm} \includegraphics[height=3cm,width=3.5cm]{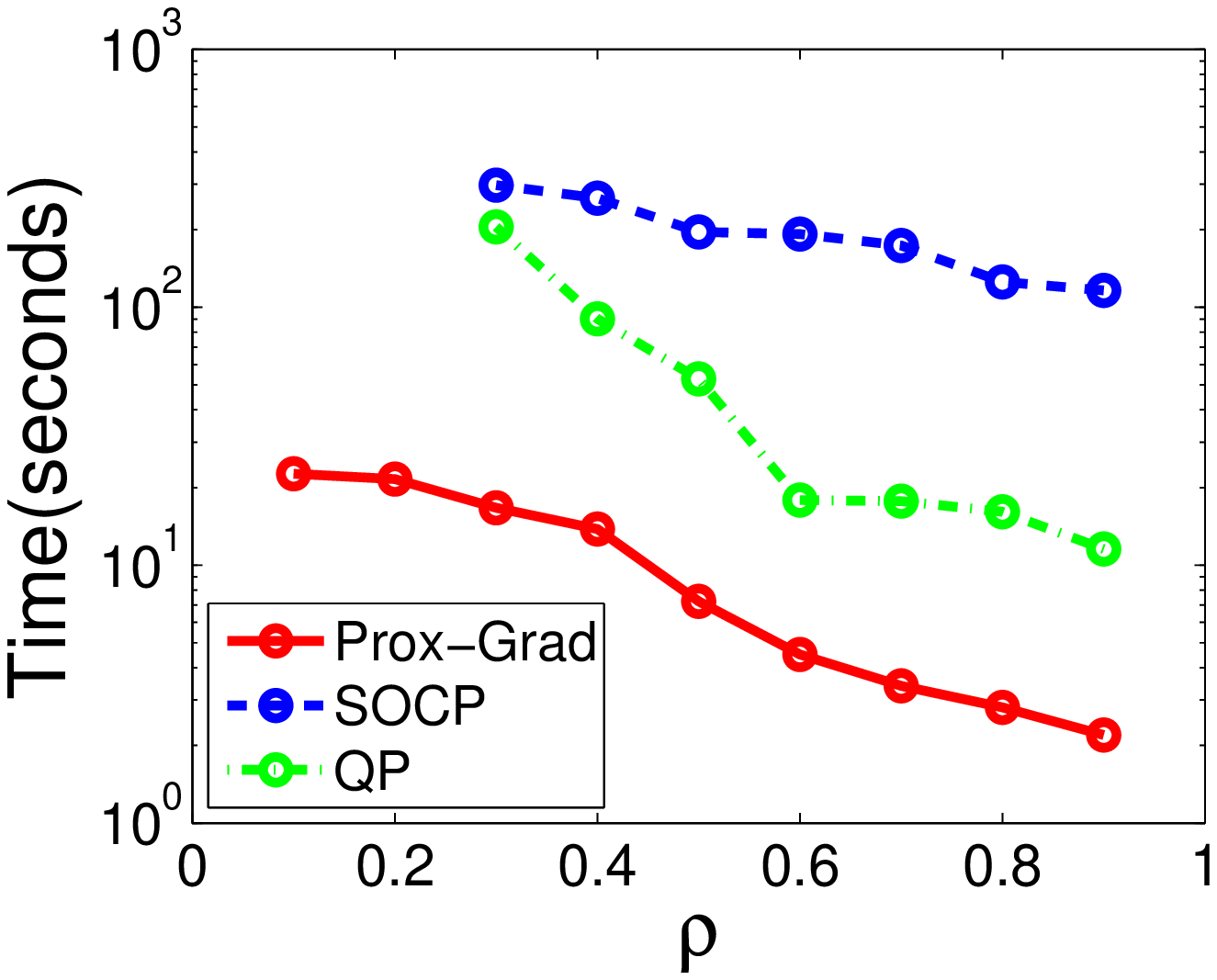}\\
   (a) & (b) & (c) & (d)
  \end{tabular}
  \caption{Comparisons of scalabilities of various optimization methods.
    For Prox-Grad, SCOP and QP, we
    (a) vary $J$ from $50$ to $500$ with a step size of $50$
    and then from $1000$ to $10,000$ with a step size of $1000$,
    fixing $N=1000, K=50$ and $\rho=0.5$,
    (b) vary $N$ from $500$ to $10000$ with a step size of $500$,
    fixing $J=100, K=50$ and $\rho=0.5$,
    (c) vary $K$ from $50$ to $500$ with a step size of
    50 and then from $1000$ to $10,000$ with a step size of $1000$,
    fixing $N=500, J=100$ and $\rho=0.5$,
    and
    (d) vary $\rho$ from 0.1 to 0.9 with a step size of $0.1$,
    fixing $N=500, J=100$ and $K=50$.
    Note that the y-axis denotes the computation time in seconds in \emph{log-scale}.}
  \label{fig:syn_scale}
  \end{small}
\end{figure}

\subsection{Asthma Dataset}

We apply GFlasso to 34 genetic markers and 53 clinical
phenotyes collected from 543 asthma patients as a part of the
Severe Asthma Research Program (SARP) \cite{sarp:2007}, and compare the
results with the ones from lasso and \lmult-regularized regression. Figure \ref{fig:sarp}(a) shows
the correlation matrix of the phenotypes after reordering the
variables using the agglomerative hierarchical clustering algorithm
so that highly correlated phenotypes are clustered with a block
structure along the diagonal. Using the threshold $\rho=0.7$, we fit
the standard lasso, \lmult-regularized multi-task regression, and GFlasso,
and show the estimated $\bm{\beta}_k$'s in Figures
\ref{fig:sarp}(c)-(e), with rows and columns representing phenotypes
and genotypes respectively. The phenotypes in rows are rearranged
according to the ordering given by the agglomerative hierarchical
clustering so that each row in Figures \ref{fig:sarp}(c)-(e) is
aligned with the phenotypes in the correlation matrix in Figure
\ref{fig:sarp}(a). We can see that the vertical bars in the GFlasso estimate in Figure
\ref{fig:sarp}(e) span the subset of highly correlated
phenotypes that correspond to blocks in Figure \ref{fig:sarp}(a).
This block structure is much weaker in the results from the lasso in
Figure \ref{fig:sarp}(c), and the blocks tend to span the entire set
of phenotypes in the results from the \lmult-regularized multi-task regression
in Figure \ref{fig:sarp}(d).



\begin{figure}[!th]
\centering
\begin{tabular}{@{}c@{}c@{}c@{}c@{}c}
\parbox[l]{3.0cm}{
\vspace{4pt}
\includegraphics[scale = 0.65]{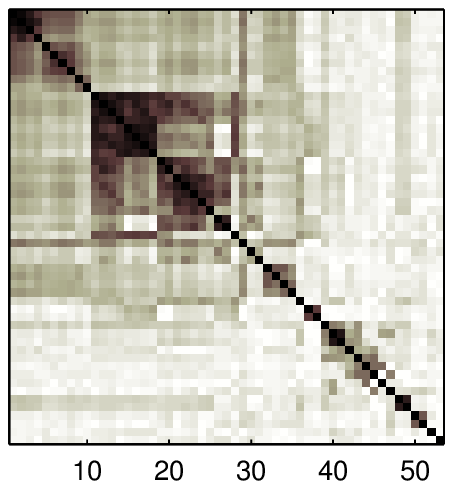}} &
\parbox[l]{3.0cm}{
\vspace{4pt}
\includegraphics[scale = 0.65]{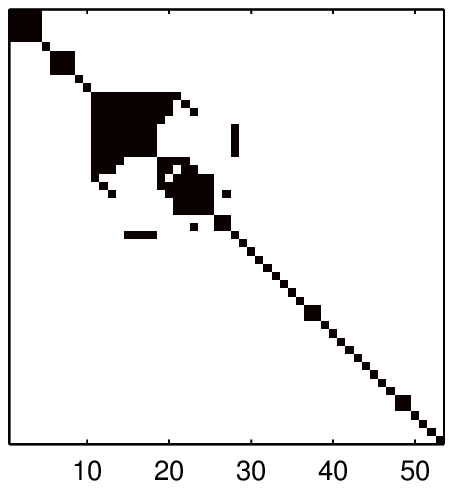}} &
\parbox[l]{2.0cm}{
\includegraphics[scale = 0.65, angle = 90]{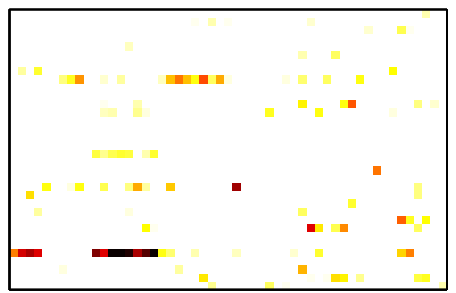}} &
\parbox[l]{2.0cm}{
\includegraphics[scale = 0.65, angle = 90]{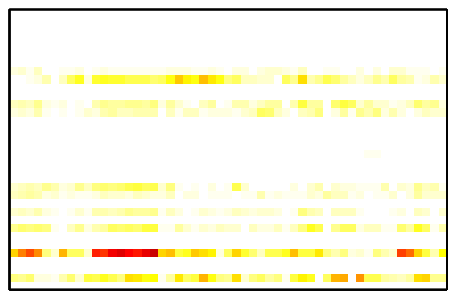}}  &
\parbox[l]{2.0cm}{
\includegraphics[scale = 0.65, angle = 90]{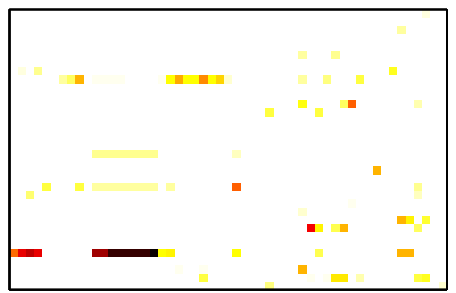}} \\
(a) & (b) & (c) & (d) & (e)
\end{tabular}
\caption{ Results for the association analysis of the asthma
dataset. (a) Phenotype correlation matrix. (b) Phenotype correlation
matrix thresholded at $\rho=0.7$. (c) lasso; (d)  \lmult-regularized
multi-task regression, and (e) GFlasso. } \label{fig:sarp}
\end{figure}

\section{Conclusions}

In this paper, we discuss a new method called GFlasso for structured
multi-task regression that exploits the correlation information in
the output variables during the estimation of regression
coefficients. GFlasso used a weighted graph structure as a guide to
find the set of relevant covariates that jointly affect highly
correlated outputs. In addition, we propose an efficient
optimization algorithm based on a proximal-gradient method that can
be used to solve GFlasso as well as any optimization problems
involving a smooth convex loss and fusion penalty defined on
arbitrary graph structures. Using simulated and asthma datasets, we
demonstrate that including richer information about output structure
as in GFlasso improves the performance for discovering true relevant
inputs, and that our proximal-gradient method is orders-of-magnitude
faster and more scalable than the standard optimization techniques
such as QP and SOCP.

\section*{Appendix}
\renewcommand{\thesection}{\Alph{section}}
\setcounter{section}{1}
\subsection{Proof of Theorem \ref{thm:key}}
The $f_\mu(\bB)$ is a convex function since it is the maximum
of a set of functions linear in $\bB$.

For the smoothness property, let the function $d^{\ast}$ be the Fenchel
conjugate of the distance function $d$ which is defined as:
\begin{equation}
\label{eq:dstar}
  d^*(\bg)=\max_{\bA \in \mathcal{Q}} \langle \bA,\bg \rangle - d (\bA).
\end{equation}
We want to prove $d^*$ is differentiable everywhere by showing that
the subdifferential $\partial d^*$ of $d^*$ is a singleton set for
any $\bg$.

By the definition in \eqref{eq:dstar}, we have, for any $\bg$ and
any $\bA \in \mathcal{Q}$:
\begin{equation}
\label{eq:conj1}
 d^*(\bg)+d(\bA)\geq \langle \bA,\bg \rangle,
\end{equation}
and the inequality holds as an equality if and only if
$\bA=\argmax_{\bA' \in \mathcal{Q}} \langle \bA',\bg \rangle - d
(\bA')$.

By the fact that for a convex and smooth function,
the conjugate of the conjugate of a function is the function itself
(Chapter E in \cite{Conv:01}), we have $d^{**}
\equiv d$. Then, \eqref{eq:conj1} can be written as:
\begin{equation}
\label{eq:conj2} d^*(\bg)+d^{**}(\bA)\geq \langle\bA,\bg\rangle,
\end{equation}
and the inequality holds as an equality if and only if
$\bg=\argmax_{\bg' \in \mathbb{R}^J}\langle\bA,\bg'\rangle -
d^*(\bg')$.

Since \eqref{eq:conj1} and \eqref{eq:conj2} are equivalent,  we know
that $\bA=\argmax_{\bA' \in \mathcal{Q}}\bA'^T \bg - d (\bA')$ if
and only if $\bg=\argmax_{\bg' \in
\mathbb{R}^J}\langle\bA,\bg'\rangle - d^*(\bg')$. The latter
equality implies that for any $\bg'$:
\begin{equation*}
d^*(\bg')\geq d^*(\bg)+\langle\bA,\bg'-\bg \rangle,
\end{equation*}
which further means that $\bA$ is a subgradient of $d^*$ at $\bg$ by
the definition of subgradient.

Summarizing  the above arguments, we conclude that $\bA$ is a
subgradient of $d^*$ at $\bg$ if and only if
\begin{equation}
\label{eq:my_max} \bA=\argmax_{\bA' \in \mathcal{Q}}\langle\bA', \bg
\rangle - d (\bA').
\end{equation}
Since $d$  is a strongly convex function, this maximization problem
in \eqref{eq:my_max} has a unique optimal solution, which means the
subdifferential $\partial d^*$ of $d^*$ at any point $\bg$ is a
singleton set that contains only $\bA$. Therefore, $d^*$ is differentiable
everywhere (Chapter D in \cite{Conv:01}) and $\bA$ is its gradient:
\begin{equation}
\label{eq:my_grad} \nabla d^*(\bg)=\bA=\argmax_{\bA' \in
\mathcal{Q}}\langle\bA', \bg \rangle - d (\bA').
\end{equation}

Now we return to our original problem of $f_\mu(\bB)$ and rewrite
it as:
$$f_\mu(\bB)=\max_{\bA \in \mathcal{Q}} \langle\bA, \Gamma(
\bB)\rangle - \mu d (\bA)=\mu\max_{\bA \in \mathcal{Q}} [
\langle\bA, \frac{\Gamma( \bB)}{\mu}\rangle -d(\bA)]=\mu
d^*(\frac{\Gamma( \bB)}{\mu}).$$

Utilizing \eqref{eq:my_grad} and  the chain rule, we know that
$f_\mu(\bB)$ is continuously differentiable  and its gradient takes
the following form:
\begin{eqnarray*}
\nabla f_\mu(\bB) &= &\mu \Gamma^*(\nabla d^*(\frac{\Gamma(
\bB)}{\mu}))=\mu \Gamma^* (\argmax_{\bA' \in \mathcal{Q}}
[\langle\bA', \frac{\Gamma( \bB)}{\mu}\rangle - d
(\bA')])\\&=&\Gamma^*(\argmax_{\bA' \in \mathcal{Q}}[\langle\bA',
\Gamma( \bB)\rangle - \mu d (\bA')])=\Gamma^*(\bA^*).
\end{eqnarray*}
For the proof of Lipschitz constant of $f_\mu(\bB)$, readers can
refer to \cite{Nesterov:05}.

\subsection{Proof of Lemma \ref{lem:norm}}

According to the definition of $\|\Gamma\|$, we have:
\begin{eqnarray*}
\|\Gamma\| &\equiv & \max_{\|\bB\|_F=1}\|\Gamma(\bB)\|_F =
\max_{\|\bB\|_F=1} \|(\lambda\bB,\gamma\bB H)\|_F
\\
&=&\max_{\|\bB\|_F=1}\sqrt{\lambda^2\|\bB\|_F^2+\gamma^2\|\bB
H\|_F^2} = \max_{\|\bB\|_F=1}\sqrt{\lambda^2+\gamma^2\|\bB H\|_F^2}
\end{eqnarray*}

Therefore, to bound $\|\Gamma\|$, we only need to find an upper
bound for $\max_{\|\bB\|_F=1}\|\bB H\|_F^2$.

According to the formulation of matrix $H$, we have
\begin{equation}
\label{app1} \|\bB H\|_F^2=\sum_{e=(m,l)\in
E}(\tau(r_{ml}))^2\sum_{j}(\beta_{jm}-\mbox{sign}(r_{ml})\beta_{jl})^2
\end{equation}

It is well known that $(a-b)^2\leq2a^2+2b^2$ and the inequality
holds as equality if and only if $a=-b$. Using this simple
inequality, for each edge $e=(m,l)\in E$, the summation
$\sum_{j}(\beta_{jm}-\mbox{sign}(r_{ml})\beta_{jl})^2$ is upper-bounded by
$\sum_{j}(2\beta_{jm}^2+2\beta_{jl}^2)=2\|\bb_m\|^2_2+2\|\bb_l\|^2_2$.
Here, the vectors $\bb_m$ and $\bb_l$ are the $m$-th and $l$-th
columns of $\bB$. The right-hand side of (\ref{app1}) can be further
bounded as:
\begin{equation*}
\begin{array}{ll}
\|\bB H\|_F^2&\leq\sum_{e=(m,l)\in E}2(\tau(r_{ml}))^2(\|\bb_m\|^2_2+\|\bb_l\|^2_2)\\
&=\sum_{k \in V}(\sum_{e \text{ incident on }k}2(\tau(r_{e}))^2)\|\bb_k\|^2_2\\
&=\sum_{k \in V} 2d_k\|\bb_k\|^2_2,
\end{array}
\end{equation*}
where $d_k$ is defined in \eqref{eq:d}. Note that the first
inequality is tight, and that the first equality can be obtained simply
by changing the order of summations.

By the definition of Frobenius norm,
$\|\bB\|_F^2=\sum_k\|\bb_k\|_2^2$. Hence, $$\max_{\|\bB\|_F=1}\|\bB
H\|_F^2\leq\max_{\|\bB\|_F=1}\sum_k
2d_k\|\bb_k\|^2_2=2\max_{k}d_k,$$ where the maximum is achieved by
setting the $\bb_k$ corresponding to the largest $d_k$ to be a unit
vector and setting other $\bb_k$'s to be zero vectors.

In summary,  $\|\Gamma\|$ can be tightly upper bounded as:
\begin{equation*}
\begin{array}{ll}
\|\Gamma\|&=\max_{\|\bB\|_F=1}\|\Gamma(\bB)\|_F\\
&=\max_{\|\bB\|_F=1}\sqrt{\lambda^2+\gamma^2\|\bB H\|_F}\\
&=\sqrt{\lambda^2+\gamma^2\max_{\|\bB\|_F=1}\|\bB H\|_F}\\
&\leq\sqrt{\lambda^2+2\gamma^2\max_{k}d_k}\equiv \|\Gamma\|_U.
\end{array}
\end{equation*}

\subsection{Proof of Theorem \ref{thm:complexity}}
Based on Theorem 2 in \cite{Nesterov:05}, we have the following
lemma:
\begin{lemma}
\label{lem:smooth} Assume that function $\tilde{f}(\bB)$ is an
arbitrary convex smooth function and its gradient $\nabla
\tilde{f}(\bB)$ is Lipschitz continuous with the Lipschitz constant
$L$ that is further upper-bounded by $L_U$. Apply Algorithm 1 to
minimize $\tilde{f}(\bB)$ and let $\bB^t$ be the approximate
solution at the $t$-th iteration. For any $\bB$, we have the
following bound:
\begin{equation}
  \label{eq:smoothbound}
  \tilde{f}(\bB^t)-\tilde{f}(\bB) \leq \frac{2L_U\|\bB\|_F^2}{t^2}.
\end{equation}
\end{lemma}

Based on Lemma \ref{lem:smooth}, we present our proof.  Recall that
the smooth approximation of the function $f(\bB)$, $\tilde{f}(\bB)$,
is defined as:
\begin{equation*}
    \tilde{f}(\bB)\equiv \frac{1}{2}
\|\bY-\bX\bB\|_F^2 + f_\mu(\bB) = \frac{1}{2} \|\bY-\bX\bB\|_F^2 +
\max_{\bA \in \mathcal{Q}} \langle \bA,\bB C \rangle -
\frac{1}{2}\|\bA\|_F^2.
\end{equation*}
Since Algorithm 1 optimizes the smooth function
$\tilde{f}(\bB)$, according to Lemma \ref{lem:smooth}, we have
\begin{equation}
  \label{eq:smoothbound1}
  \tilde{f}(\bB^t)-\tilde{f}(\bB^{\ast}) \leq \frac{2L_U\|\bB^{\ast}\|_F^2}{t^2},
\end{equation}
where $L_U=\lambda_{\max} (\bX^T\bX) + \frac{\|\Gamma\|_U^2}{\mu}$
is the upper bound of the Lipschitz constant for $\nabla
\tilde{f}(\bB)$.

We want to utilize the bound in \eqref{eq:smoothbound1}; so we
decompose $f(\bB^t)-f(\bB^{\ast})$ into three terms:
\begin{equation}
\label{eq:decompose}
  f(\bB^t)-f(\bB^{\ast})=\left(f(\bB^t)-\tilde{f}(\bB^t)\right) + \left( \tilde{f}(\bB^t)
  -\tilde{f}(\bB^{\ast})\right) + \left( \tilde{f}(\bB^{\ast}) -  f(\bB^{\ast})\right).
\end{equation}
According to the definition of $\tilde{f}$, we know that for any
$\bB$,
\begin{equation*}
  \tilde{f}(\bB) \leq f(\bB)  \leq   \tilde{f}(\bB) + \mu D,
\end{equation*}
where $D \equiv \max_{\bA \in \mathcal{Q}}d(\bA)$. Therefore, the
first term in \eqref{eq:decompose}, $f(\bB^t)-\tilde{f}(\bB^t)$, is
upper-bounded by  $\mu D$; and the last term in \eqref{eq:decompose}
is less than or equal to 0, i.e. $\tilde{f}(\bB^{\ast}) -
f(\bB^{\ast}) \leq 0$. Combining \eqref{eq:smoothbound1} with these
two simple bounds, we have:
\begin{equation}
\label{eq:bound1}
   f(\bB^t)-f(\bB^{\ast}) \leq \mu D +
   \frac{2L\|\bB^{\ast}\|_F^2}{t^2} \leq  \mu D +
   \frac{2\|\bB^{\ast}\|_F^2}{t^2} \left(\frac{\|\Gamma\|_U^2}{\mu}+\lambda_{\max}(\bX^T\bX) \right).
\end{equation}
By setting $\mu=\frac{\epsilon}{2D}$ and plugging it into the right
hand side of \eqref{eq:bound1}, we obtain
\begin{equation}
\label{eq:bound2}
 f(\bB^t)-f(\bB^{\ast}) \leq  \frac{\epsilon}{2}
 +\frac{2\|\bB^{\ast}\|_F^2}{t^2}\left(\frac{2D\|\Gamma\|_U^2}{\epsilon}+\lambda_{\max}\left(\bX^T\bX\right)\right).
\end{equation}
If we require the right-hand side of \eqref{eq:bound2} to be equal to
$\epsilon$ and solve for $t$, we obtain the bound of $t$ in
\eqref{eq:bound}.

Note that we can set $\mu=\frac{\epsilon}{h}$ for any $h>1$ to
achieve $O\left(\frac{1}{\epsilon}\right)$ convergence rate,
which is different from \eqref{eq:bound} only by a constant factor.

\subsection{Proof of Theorem \ref{thm:statconvergence}}

Define $V_N(\mathbf{U})$ by
    \begin{eqnarray}
    V_N(\mathbf{U}) =
    \sum_{k=1}^K \sum_{i=1}^N
        \big[(\varepsilon_{ik}-\mathbf{u}_k^T\mathbf{x}_i/\sqrt{N})^2-\varepsilon_{ik}^2\big]
        + \lambda_N^{(1)} \sum_{k} \sum_{j}
        \big[|\beta_{jk}+u_{jk}/\sqrt{N}|-|\beta_{jk}|\big] \nonumber \\
        + \lambda_N^{(2)} \sum_{(m,l)} f(r_{m,l}) \sum_j
        \big[|\beta_{j,(m,l)}'+
        u_{j,(m,l)}'/\sqrt{N}| -|\beta_{j,(m,l)}'|\big], \nonumber
    \end{eqnarray}
Note that $V_N(\mathbf{u})$ is minimized at
$\sqrt{N}(\hat{\mathbf{B}}_N-\mathbf{B})$. Notice that we have
    \begin{subequations}
    \begin{align}
    \sum_{k=1}^K \sum_{i=1}^N
    \Big[(\varepsilon_{ik}-\mathbf{u}_k^T\mathbf{x}_i/\sqrt{N})^2
        -\varepsilon_{ik}^2\Big] \to_{d}
        -2\sum_k [ \mathbf{u}_k^T \mathbf{W}+\mathbf{u}_k^TC\mathbf{u}_k],
        \quad\quad
        \quad\quad\quad\quad\quad\quad
        \quad\quad\quad\quad\quad\quad
        \quad\quad\quad\quad\quad\quad
        \nonumber \\
    \lambda_N^{(1)} \!\sum_{k} \sum_{j}
        \big[|\beta_{jk}\!+\!u_{jk}/\sqrt{N}|-|\beta_{jk}|\big]
        \to_{\substack{d}}
        \lambda_0^{(1)} \sum_{k} \sum_{j} \big[u_{jk}
        \textrm{sign}(\beta_{jk})I(\beta_{jk} \neq 0)
        \!+\!|u_{jk}|I(\beta_{jk}=0)\big],
        \quad\quad\quad\quad\quad\quad\quad
    \,
        \nonumber
        \\
    \gamma_N^{(2)} \sum_{e=(m,l)\in E} f(r_{ml}) \sum_j
        \big[|\beta_{je}'+
        u_{je}'/\sqrt{N}|
        -|\beta_{je}'|\big]
        \quad\quad\quad\quad\quad\quad\quad
        \quad\quad\quad\quad\quad\quad\quad
        \quad\quad\quad\quad\quad\quad
        \quad
        \quad\quad\quad\quad\quad\quad \nonumber \\
        \to_{d}
        \lambda_0^{(2)} \sum_{e=(m,l)\in E} f(r_{ml}) \sum_j
        \big[ u_{je}' \mysgn( \beta_{je}')
        I(\beta_{je}' \neq 0 )
        +|u_{je}'|I(\beta_{je}'=0)\big].
        \quad\quad\quad\quad\quad\quad
        \quad\quad\quad\quad\quad
        \nonumber
    \end{align}
    \end{subequations}
Thus, $V_N(\mathbf{U})\to_{\substack{d}} V(\mathbf{U})$ with the
finite-dimensional convergence holding trivially. Since $V_N$ is
convex and $V$ has a unique minimum, it follows that
$\arg\!\min_\mathbf{U}
V_N(\mathbf{U})=\sqrt{N}(\hat{\mathbf{B}}-\mathbf{B})\to_{\substack{d}}$
$\arg\!\min_{\mathbf{U}} V(\mathbf{U})$.

\newpage
\bibliography{paper}
\bibliographystyle{plain}
\end{document}